\documentclass[12pt]{article}
\usepackage{comment}
\usepackage{amsfonts}
\usepackage{amsmath}
\usepackage{amsthm}
\usepackage{amssymb}
\usepackage{multirow}
\usepackage{multicol}
\usepackage{mathrsfs} 
\usepackage[margin = 1in]{geometry}
\usepackage{hyperref}
\usepackage[symbol]{footmisc}

\usepackage[title]{appendix}
\hypersetup{
    colorlinks=true,
    linkcolor=blue,
    filecolor=magenta,      
    urlcolor=blue,
    citecolor= OliveGreen
}
\usepackage{times}
\usepackage{graphicx}
\usepackage{subcaption}
\providecommand{\keywords}[1]{\textbf{Keywords:} #1}
\usepackage{bbm}
\usepackage{natbib}
\usepackage{enumitem}
\usepackage[dvipsnames]{xcolor}
\usepackage[ruled,vlined]{algorithm2e}
\usepackage{authblk}

\newtheorem{theorem}{Theorem}
\newtheorem{lemma}{Lemma}[section]
\newtheorem{proposition}{Proposition}
\newtheorem{definition}{Definition}[section]
\newtheorem{remark}{Remark}

\newtheorem{assumption}{Assumption}[section]

\def\abs#1{\left|#1\right|}

\def\argmax{{\arg\max}}

\def\br{\mathbf{r}}

\def\bbC{\mathbb{C}}

\def\bbE{\mathbb{E}}

\def\bbI{\mathbb{I}}

\def\bbN{\mathbb{N}}

\def\bbR{\mathbb{R}}
\def\bbS{\mathbb{S}}

\def\cB{\mathcal{B}}

\def\cE{\mathcal{E}}
\def\cF{\mathcal{F}}
\def\cG{\mathcal{G}}
\def\cH{\mathcal{H}}

\def\cL{\mathcal{L}}
\def\cM{\mathcal{M}}
\def\cN{\mathcal{N}}

\def\cP{\mathcal{P}}
\def\cQ{\mathcal{Q}}
\def\cR{\mathcal{R}}

\def\cT{\mathcal{T}}

\def\xmax{x_{\sf max}}
\def\bmax{b_{\sf max}}

\def\diag{{\sf diag}}

\def\ind{{\sf ind}}

\def\lspan{{\sf span}}

\def\norm#1{\left\|#1\right\|}

\def\Beta{{\sf Beta}}

\def\supp{{\sf supp}}

\def\comp{{\sf comp}}

\def\Lap{{\sf Lap}}

\def\pr{\mathbb{P}}

\def\R{{\mathbb{R}}}

\def\his{\mathcal{H}}
\def\calS{\mathcal{S}}

\def\sfN{{\mathsf{N}}}

\def\ind{\mathbbm{1}}

\usepackage{amsmath}
\newcommand{\innerprod}[1]{\left\langle#1\right\rangle}

\mathchardef\mhyphen="2D
\newcount\Comments  
\usepackage{color}
\newcommand{\kibitz}[2]{\ifnum\Comments=1\textcolor{#1}{#2}\fi}
\usepackage{color}
\definecolor{darkgreen}{rgb}{0,0.5,0}
\definecolor{purple}{rgb}{1,0,1}

\title{Thompson Sampling for High-Dimensional Sparse Linear Contextual Bandits}

\author{
   Sunrit Chakraborty 
 {\color{blue}*}
   \quad Saptarshi Roy {\color{blue}*}
   \quad Ambuj Tewari\\
   University of Michigan, Ann Arbor\\
   \texttt{\{sunritc, roysapta, tewaria\}@umich.edu}\\
   }

\begin{document}
\maketitle
\def\thefootnote{{\color{blue}*}}\footnotetext{Alphabetical order: First two authors have equal contribution.}\def\thefootnote{\arabic{footnote}}

\begin{abstract}
We consider the stochastic linear contextual bandit problem with high-dimensional features. We analyze the Thompson sampling algorithm using special classes of sparsity-inducing priors (e.g., spike-and-slab) to model the unknown parameter and provide a nearly optimal upper bound on the expected cumulative regret. To the best of our knowledge, this is the first work that provides theoretical guarantees of Thompson sampling  in high-dimensional and sparse contextual bandits. For faster computation, we use variational inference instead of Markov Chain Monte Carlo (MCMC) to approximate the posterior distribution. Extensive simulations demonstrate the improved performance of our proposed algorithm over existing ones. 
\end{abstract}


\section{Introduction}
\label{sec:introduction}
Sequential decision-making, including bandits problems and reinforcement learning, has been one of the most active areas of research in machine learning. It formalizes the idea of selecting actions based on current knowledge to optimize some long term reward over sequentially collected data. On the other hand, the abundance of personalized information allows the learner to make decisions while incorporating this contextual information, a setup that is mathematically formalized as contextual bandits. 
Moreover, in the big data era, the personal information used as contexts often has a much larger size, which can be modeled by viewing the contexts as high-dimensional vectors. Examples of such models cover internet marketing and treatment assignment in personalized medicine, among many others. 

A particularly interesting special case of the contextual bandit problem is the linear contextual bandit problem, where the expected reward is a linear function of the features \citep{abe2003reinforcement, auer2002using}. Under this setting, \cite{dani2008stochastic}, \cite{chu2011contextual} and \cite{abbasi2011improved} showed polynomial dependence of the cumulative regret on ambient dimension $d$ and time horizon $T$ in low dimensional case. {\color{black}Specifically, \cite{dani2008stochastic} and \cite{abbasi2011improved} proved a regret upper bound scaling as $O(d \sqrt{T})$, while \cite{chu2011contextual} showed a regret upper bound of the order $O(\sqrt{d T})$.} It is worthwhile to mention that all of the aforementioned algorithms fall under a certain class of algorithms known as upper confidence bound (UCB) type algorithms that rely on the construction of a specific confidence set for the unknown parameter. In contrast, Thompson Sampling (TS) maintains uncertainty about the unknown parameter in the form of a posterior distribution. The first TS algorithm under this setting was proposed by \cite{agrawal2013thompson} where they established a regret bound of the order $O(d^2 \delta^{-1} \sqrt{T^{1+\delta}})$ for any $\delta \in (0,1)$.

There is also a large body of work present in high-dimensional
sparse linear contextual bandit setup, where the reward only
depends on a small subset of features of the observed contexts.  This area has recently attracted considerable attention
due to its abundance in modern reinforcement learning
applications (e.g, clinical trials, personalized recommendation systems, etc.) and has quite naturally spawned theoretical research in this direction. To mention a few, the LASSO-bandit algorithm in \cite{bastani2020online}, doubly robust-LASSO \citep{kim2019doubly} and MCP-bandit algorithm in \cite{wang2018minimax} are shown to have a regret bound with logarithmic dependence on $d$. \cite{HaoLattimore2020} proposed \emph{explore the sparsity and then commit} (ESTC) algorithm which also enjoys the regret scaling as $O(T^{2/3} \log d)$ in ``data poor'' regime, i.e., when $T\ll d$. Very recently, \cite{chen2022nearly} proposed sparse-LinUCB algorithm and sparse-SupLinUCB algorithm for high dimensional contextual bandits with linear rewards. Both of the algorithms are based on the best subset selection approach and enjoy $O(\sqrt{T})$ scaling in the time horizon and poly-logarithmic dependence in the ambient dimension $d$. However, there has been very limited work dedicated to analyzing TS algorithms in high-dimensional sparse bandit setups. \cite{hao2021information} proposed a sparse information-directed sampling (IDS) algorithm which under a special case reduces to a TS algorithm based on a spike-and-slab Gaussian-Laplace prior. However, the regret bound of IDS scales polynomially in $d$, which is sub-optimal in the high-dimensional regime. In related work, \cite{gilton2017sparse} proposed a linear TS algorithm based on a relevance vector machine (RVM) which again suffers from sub-optimal dependence on $d$.

In this paper, we specifically focus on the high-dimensional
sparse linear contextual bandit (SLCB) setup and propose a TS algorithm based on a sparsity-inducing prior that enjoys almost dimension-independent regret bound. While TS algorithms have been known to empirically perform better than optimism-based algorithms \citep{chapelle2011empirical, kaufmann2012thompson}, theoretical understanding of these is challenging due to the complex dependence structure of the bandit environment. Moreover, posterior sampling in high-dimensional regression (which is the crucial step for TS in high-dimensional SLCB), using MCMC, generally suffers from computational bottleneck.
Our work overcomes all these challenges and makes  the following contributions:
\vspace{-3mm}
\begin{enumerate}
\itemsep0em
    \item We use the sparsity inducing prior proposed in \cite{castillo2015bayesian} for posterior sampling and establish posterior contraction result for \emph{non-i.i.d.\ observations} coming from bandit environment and for a wide class of noise distributions.
    \item Using the posterior contraction result, we establish an almost \emph{dimension free} regret bound for our proposed TS algorithm under different arm-separation regimes parameterized by $\omega$. The algorithm enjoys minimax optimal performance for $\omega \in [0,1)$. To the best of our knowledge, this is the first work that proposes a novel TS algorithm with desirable regret guarantees in high-dimensional and sparse SLCB setup.
    \item Our algorithm does {\em not} need the knowledge of model sparsity level, unlike other algorithms such as LASSO-bandit, MCP-bandit, ESTC, etc.
    \item Finally, the prior allows us to design a computationally efficient TS algorithm based on Variational Bayes. 
    
\end{enumerate}

The rest of the paper is organized as follows. In Section \ref{sec:problem}, we introduce the problem formally and discuss the assumptions and prior distribution. In Section \ref{sec:main_results}, we present the crucial posterior contraction result and the main regret bound for our proposed algorithm. In Section \ref{sec:computation}, we discuss the challenges of drawing samples from the posterior in such problems and present a faster alternative relying on variational inference. In Section \ref{sec:numerical}, we present simulation studies under different setups comparing our proposed method with existing algorithms. Detailed proofs of results and technical lemmas are deferred to the appendix.

\textbf{Notation:} Let $\R$ and $\R^+$ denote the set of real numbers and the set of non-negative real numbers respectively. Denote by $\R^p$ the $p$-dimensional Euclidean space and $\bbS^{p-1}:=\{x\in\bbR^p \mid \norm{x}_2=1\}$ the $(p-1)$-dimensional unit sphere. For a positive integer $K$, denote by $[K]$ the set $\{1, 2, \ldots, K\}$. 
Regarding vectors and matrices, for a vector $v \in \R^p$, we denote by $\norm{v}_0, \norm{v}_1,\norm{v}_2, \norm{v}_\infty$ the $\ell_0, \ell_1, \ell_2, \ell_\infty$ norms of $v$ respectively. We use $\bbI_{p}$ to denote the $p$-dimensional identity matrix. For a $p\times q$ matrix $M$  we define the $\norm{M}:= \max_{j \in [q]} \sqrt{(M^\top M)_{j,j}}$. 

 Regarding distributions, $\sfN(\mu,\sigma)$ denotes the Gaussian distribution with mean $\mu$ and standard deviation $\sigma$, $\Lap(\lambda)$ denotes the Laplace distribution with density $f(x)= (\lambda/2)\exp(-\lambda\abs{x})$, and $\Beta(a,b)$ denotes the beta distribution with parameters $a,b$.

Throughout the paper, let $O(\cdot)$ denote the standard big-O notation, i.e., we say $a_n = O(b_n)$ if there exists a universal constant $C>0$, such that $a_n \leq C b_n$ for all $n \in \mathbb{N}$. Sometimes for notational convenience, we write $a_n \lesssim b_n$ in place of $a_n = O(b_n)$. We write $a_n \asymp b_n$ or $a_n = \Theta(b_n)$ if $a_n = O(b_n)$ and $b_n = O(a_n)$.

\section{Problem Formulation}
\label{sec:problem}
We consider a linear stochastic contextual bandit with $K$ arms. At time $t\in [T]$, context vectors $\{x_i(t)\}_{i\in [K]}$ are revealed for every arm $i$. We assume $x_i(t)\in \bbR^d$ for all $i\in[K], t\in[T]$ and for every $i$, $\{x_i(t)\}_{t\in[T]}$ are i.i.d.\ from some distribution $\cP_i$. At every time step $t$, an action $a_t \in [K]$ is chosen by the learner and a reward $r(t)$ is generated according to the following linear model:
\begin{align}
\label{eq: base model}
    r(t) = x_{a_t}(t)^\top \beta^* + \epsilon(t)
\end{align}
where $\beta^*\in\bbR^d$ is the unknown true signal and $\{\epsilon(t)\}_{t\in[T]}$ are independent sub-Gaussian random noise, also  independent of all the other processes. We assume that the true parameter $\beta^*$ is $s^*$--sparse, i.e., $\norm{\beta^*}_0 = s^*$. We denote by $S^*$ the true support of $\beta^*$, i.e., $S^* = \{ j : \beta_j^* \neq 0\}$.

The goal is to design a sequential decision-making policy $\pi$ that maximizes the expected cumulative reward over the time horizon. To formalize the notion, we define the history $\cH_t$ up to time $t$ as follows:
$$\cH_t := \left\{(a_\tau, r(\tau), \{x_i(\tau)\}_{i\in[K]}) : \tau \in [t] \right\},$$
and an admissible policy $\pi$ generates a sequence of random variables $a_1, a_2,\dots$ taking values in $[K]$ such that $a_t$ is measurable with respect to the $\sigma$-algebra generated by the previous feature vectors from each arm, observed rewards of the chosen arms till the previous round and the current feature vectors, i.e., measurable with respect to the filtration
$\cF_t := \sigma\left(x_{a_{\tau}}(\tau), r(\tau), x_{i}(t); \tau\in [t-1], i\in[K]\right).$

Thus, an algorithm for contextual bandits is a policy $\pi$, which at every round $t$, chooses an action (arm) $a_t$ based on history $\cH_{t-1}$ and current contexts. We note that although contexts of the previous round corresponding to arms that were not chosen are in $\cF_t$, however, they do not provide useful information on the parameter, since we do not observe rewards corresponding to them under the bandit feedback, and hence are not included in the history $\cH_t$. To measure the quality of performance, we compare it with the oracle policy $\pi^*$ which uses the knowledge of the true $\beta^*$ to choose the optimal action $a_t^* := \argmax_{i\in[K]} x_i(t)^\top \beta^*$.
Define $\Delta_i(t)$ to be the difference between the mean rewards of the optimal arm and $i$th arm at time $t$, i.e.,
$\Delta_i(t) = x_{a_t^*}(t)^\top\beta^* - x_i(t)^\top \beta^*.$
Note that under the random-design assumption, $a_t^*$ is also random. Then the regret at time $t$ is defined as $\text{regret}(t) = \Delta_{a_t}(t)$ and the objective of the learner is to minimize the total regret till time $T$, defined as 
$
\label{eq: regret definition}
    R(T) = \sum_{t\in[T]} \text{regret}(t).
$
We also define the matrix $X_t := (x_{a_1}(1), \ldots, x_{a_t}(t))^\top$. 
The time horizon $T$ is finite but possibly unknown, but much smaller compared to the ambient dimension of the parameter, i.e. $d \gg T$. \cite{hao2021information} refers to this regime as ``data-poor" regime; such a regime adds an extra layer of hardness on top of the difficulty incurred by the sparse structure of $\beta$. We also assume that $K$ is fixed and much smaller compared to both $d$ and $T$.

\subsection{Assumptions}
\label{subsec:assumptions}
In this section, we discuss the assumptions of our model. 


\begin{definition}[Sparse Riesz Condition (SRC)]
\label{def: sparse eigen value}
Let $M$ be a $d \times d$ positive semi-definite matrix. The maximum and minimum sparse eigenvalues of M with parameter $s\in [d]$ are defined as follows: 
\begin{align*}
    \label{eq: sparse-eigen}
    \phi_{\min}(s; M)&:= \inf_{\delta: \delta \neq 0, \norm{\delta}_0 \leq s}\frac{\delta^\top M \delta}{\norm{\delta}_2^2}, \\ \phi_{\max}(s; M)&:= \sup_{\delta: \delta \neq 0, \norm{\delta}_0 \leq s}\frac{\delta^\top M \delta}{\norm{\delta}_2^2}.
\end{align*}
We say $M$ has Sparse Riesz Condition if  $0< \phi_{\min}(s, M) \leq \phi_{\max}(s; M) < \infty$.
\end{definition}

 Now we are ready to state the assumptions on the context distributions, which are as follows:

\begin{assumption}[Assumptions on Context Distributions] \label{assumptions: contexts}
We assume that
\vspace{-5pt}
\begin{enumerate}[label=(\alph*)]
\itemsep0em
    \item \label{item: context_bound} For some constant $\xmax\in \bbR^+$, we have that for all $i\in[K]$, $\cP_i (\norm{x}_\infty \leq \xmax) =1$.
    \item \label{item: context_subgaussian} For all arms $i \in [K]$ the distribution $\cP_i$ is sub-Gaussian, i.e.,
    $
    \max_{i \in [K]}\; \bbE_{x \sim \cP_i} \{\exp(t u^\top x)\} \leq \exp(\vartheta^2 t^2/2),
    $
    for all $t \in \bbR$ and $u\in \bbS^{d-1}$. 
    \item \label{item: anti-concentration} There exists a  constant $\xi\in \bbR^+$ such that for each $u\in \bbS^{d-1}\cap \{v\in\bbR^d : \norm{v}_0 \leq Cs^*\}$ and $h\in \bbR^+$ 
    $\cP_i(\innerprod{x, u}^2 \leq h) \leq \xi h,$
    for all $i\in [K]$, ,where $C\in (2,\infty)$.
    \item \label{item: sparse_eigenvalue}The matrix $\Sigma_i := \bbE_{x\sim \cP_i} [x x^\top]$ has bounded maximum sparse eigenvalue, i.e., 
    $
\phi_{\max}(Cs^*, \Sigma_i)
        \leq \phi_u   < \infty,
    $
    for all $i \in [K]$,
     where $C$ is the same constant as in part (c).
\end{enumerate}
\end{assumption}

Assumption \ref{assumptions: contexts}\ref{item: context_bound} basically tells that the contexts are bounded; such assumptions are standard in the bandit literature to obtain results on regret bound that are independent of the scaling of the contexts (or parameter). Assumption \ref{assumptions: contexts}\ref{item: context_subgaussian} says that all the arm-contexts are generated from \emph{mean zero} sub-Gaussian distribution with parameter $\vartheta$, for all time point $t$. This is indeed a very mild assumption on the context distribution and a broad class of distributions enjoys such property. For example, truncated multivariate normal distribution with covariance matrix $\bbI_d$, where the truncation is over the set $\{u \in \bbR^d : \norm{u}_\infty \leq 1\}$ is a valid distribution for the contexts.
 Assumption \ref{assumptions: contexts}\ref{item: anti-concentration} talks about anti-concentration condition that plays a critical role in controlling the estimation accuracy of $\beta^*$. Intuitively, this condition prohibits the context features to fall along a singular direction. In particular, this condition ensures that the directions of the arms are well spread across every direction and thus implicitly promotes exploration in the feature space.
 Moreover, this anti-concentration is very mild as it holds if the density of $x_i(t)^\top u$ is bounded above.
 Assumption \ref{assumptions: contexts}\ref{item: sparse_eigenvalue} imposes an upper bound on the maximum sparse eigenvalue of $\Sigma_i$ which is a common assumption in high-dimensional literature \citep{zhang2008sparsity,zhang2010nearly}.
 

 Next, we come to the assumptions on the true parameter $\beta^*$
 \begin{assumption}[Assumptions on the true parameter]
\label{assumptions: arm-separation}
We assume the followings:
\vspace{-5pt}
\begin{enumerate}[label = (\alph*)]
\itemsep0em
    \item \label{item: beta-l1} Sparsity and Soft-sparsity: There exist positive constants $s^*\in \bbN$ and $\bmax \in \bbR^+$  such that $\norm{\beta^*}_0 = s^*$ and $\norm{\beta^*}_1 \leq \bmax$.
    
    \item \label{item: margin-cond} Margin condition: There exists positive constants $\Delta_*, A$ and $\omega\in [0, \infty]$, such that for $h \in \left[ A \sqrt{\log(d)/ T}, \Delta_*\right]$ and for all $t\in [T]$,
    \[ \pr\left( x_{a^*_t}(t)^\top \beta^* \leq \max_{i \neq a^*_t} x_{i}(t)^\top \beta^* +h\right) \leq  \left(\frac{h}{\Delta^*}\right)^\omega.\]
\end{enumerate}
\end{assumption}
The first part of the assumption requires boundedness of the true parameter $\beta^*$ to make the final regret bound scale free. Such an assumption is also standard in bandit literature \citep{bastani2020online, abbasi2011improved}.

The second part of the assumption imposes a margin condition on the arm distributions. Essentially, this assumption controls the probability of the optimal arm falling into $h$-neighborhood of the sub-optimal arms. As $\omega$ increases, the margin condition becomes stronger as the sub-optimal arms are less likely to fall close to the optimal arms. As a result, it becomes easier for any bandit policy to distinguish the optimal arm. As an illustration, consider the two extreme cases $\omega = \infty$ and $\omega = 0$. The $\omega = \infty$ case tells that there is a deterministic gap between rewards corresponding to the optimal arm and sub-optimal arms. This is the same as the ``gap assumption'' in \cite{abbasi2011improved}. Thus, quite evidently it is easy for any bandit policy to recognize the optimal arm. This phenomenon is reflected in the regret bound of Theorem 5 in \cite{abbasi2011improved}, where the regret depends on the time horizon $T$ only though poly-logarithmic terms. In contrast, $\omega = 0$ corresponds to the case when there is no apriori information about the separation between the arms, and as a consequence, we pay the price in regret bound by a $\sqrt{T}$ term \citep{hao2021information, agrawal2013thompson, chu2011contextual}. 

The margin condition with $\omega =1$ has been assumed in \cite{goldenshluger2013linear, bastani2020online, wang2018minimax} and will be satisfied when the density of $x_i(t)^\top \beta^*$ is uniformly bounded for all $i \in [K]$. 
\cite{li2021regret} also discusses an example where the margin condition holds for different values of $\omega$. 

The final assumption is on the noise variables:
\begin{assumption}[Assumption on Noise]\label{assumptions: noise}
We assume that the random variables $\{\epsilon(t)\}_{t\in[T]}$ are independent and also independent of the other processes and each one is $\sigma-$Sub-Gaussian, i.e.,  $\bbE [e^{a \epsilon(t)}] \leq e^{\sigma^2 a^2/2}$  for all $ t \in [T]$ and $a\in \bbR$. 
\end{assumption}
Various families of distribution satisfy such a requirement, including normal distribution and bounded distributions, which are commonly chosen noise distributions. Note that such a requirement automatically implies that for every $t\in[T]$, $\bbE[\epsilon(t)] = 0$ and $\text{Var}[\epsilon(t)]\leq \sigma^2$.

\subsection{Thompson Sampling and Prior}
\label{subsec:prior}
We discuss the basics of Thompson sampling and introduce the specific structure of the prior that we use and analyze. Typically, we place a prior $\Pi$ on the unknown parameter ($\beta$ in our case) along with a \emph{specified likelihood model} on the data, and do the following: while taking action, we draw a sample from the posterior distribution of the parameter given the data and use that as the proxy for the unknown parameter value, hence in our case at time $t$, we draw a sample $\hat{\beta}_t \sim \Pi(\beta|\cH_{t-1})$ and choose $a_t = \argmax_i\ x_i(t)^\top \hat{\beta}_t$ as the action. While simple enough to describe, Thompson sampling has been difficult to analyze theoretically, particularly because of the complex dependence between the observations due to the bandit structure. The choice of prior plays a crucial role, as we shall see, in providing the correct exploration-exploitation trade-off. In the high-dimensional sparse case that we are dealing with, this choice is specifically important since we do not wish to have a linear dependence on the dimension $d$ in our regret bound - which would be incurred if we use the normal prior-likelihood setup of \cite{agrawal2013thompson}, which analyzes Thompson sampling in contextual bandits.

While there is a rich literature on Bayesian priors for high dimensional regression, including horseshoe priors and slab-and-spike priors among others, we shall be using the complexity prior introduced in \cite{castillo2015bayesian}. Specifically, we consider a prior $\Pi$ on $\beta$ that first selects a dimension $s$ from a prior $\pi_d$ on the set $[d]$, next a random subset $S\subset [d]$ of size $|S|=s$ and finally, given $S$, a set of nonzero values $\beta_S:=\{\beta_i: i\in S\}$ from a prior density $g_S$ on $\bbR^S$. Formally, the prior on $(S,\beta)$ can be written as
\begin{equation}
\label{eq: prior}
    (S,\beta) \mapsto \pi_d(|S|)\frac{1}{\binom{d}{|S|}} g_S(\beta_S)\delta_0(\beta_{S^c}),
\end{equation}
where the term $\delta_0(\beta_{S^c})$ refers to coordinates $\beta_{S^c}$ being set to 0. Moreover, we choose $g_S$ as a product of Laplace densities on $\bbR$ with parameter $\lambda/\sigma$, i.e., $\beta_i \mapsto (2\sigma)^{-1}\lambda \exp\left(-\lambda|\beta_i|/\sigma\right)$ for all $i \in S$. Note that, here we assume that the noise level $\sigma$ is known. In practice, one can add another level of hierarchy by setting a prior on $\sigma$ but in  this paper we do not pursue that direction.

The prior $\pi_d$ plays the role of expressing the sparsity of the parameter. This is in contrast to other priors like product of independent Laplace densities over the coordinates (typically known as Bayesian Lasso), where the Laplace parameter plays the role of shrinking the coefficients towards 0. However, in our case, the scale parameter $\lambda$ of the Laplace does not have this role and we assume that during the $t$th round we use $\lambda \in [(5/3)\overline{\lambda}_t , 2\overline{\lambda}_t ]$, where $\overline{\lambda}_t \asymp \sqrt{t \log d }$,  which is the usual order of the regularization parameter used in the LASSO. 


The choice of the prior $\pi_d$ is very critical; it should down weight big models but at the same time give enough mass to the true model. Following \cite{castillo2015bayesian}, we assume that there are constants $A_1, A_2, A_3, A_4>0$ such that $\forall \, s\in [d]$
\begin{equation}
\label{eq: sparsity-level prior bound}
    A_1 d^{-A_3} \pi_d(s-1) \leq \pi_d(s) \leq A_2 d^{-A_4} \pi_d(s-1).
\end{equation}
Complexity priors of the form $\pi_d(s) \propto c^{-s} d^{-as}$ for constants $a,c$ satisfy the above requirement. Moreover, slab and spike priors of the form $(1-r)\delta_0 + r \Lap(\lambda/\sigma)$ independently over the coordinates satisfy the requirement with 
hyperprior on $r$ being $\Beta(1, d^u)$.

Finally, we specify the data likelihood that is crucial for the TS algorithm. At each time point $t\in [T]$, given the observations coming from model \eqref{eq: base model}, we model the $\{\epsilon(\tau)\}_{\tau\leq t}$ as i.i.d. $\sfN(0,\sigma^2)$. We emphasize that this Gaussian assumption is \emph{only required for likelihood modeling} and our main results hold under any true error distribution satisfying Assumption \ref{assumptions: noise}. 


\section{Main Results}
\label{sec:main_results}

\subsection{Posterior contraction}
Now, we present an informal version of the main posterior contraction result  for the estimation of $\beta$. A more detailed version of the result with exact rates, along with the measure theoretic details, is in Appendix \ref{sec: proof posterior contraction}.

\begin{theorem}[Informal]
    \label{thm: informal posterior contraction}
  Write $\br_t = (r(1), \ldots, r(t))^\top, $ and let the Assumption \ref{assumptions: contexts}--\ref{assumptions: noise} hold  with $C = \Theta(\phi_u \vartheta^2 \xi K \log K)$, and $K \geq 2, d\geq T$. With  $\lambda \asymp \xmax (t \log d)^{1/2}$ and $\varepsilon_{t,d,s}=  s^*\{(\log d + \log t)/t\}\}^{1/2}$, the following holds as $t\to \infty$: 
 
\vspace{-4mm}
\begin{align*}
&
 \bbE_{\br_t} \Pi \left(\norm{\beta - \beta^*}_1 \gtrsim \sigma \varepsilon_{t,d,s}\; \bigg \vert\; \br_t, X_t\right) \overset{a.s.}{\to} 0.
\end{align*}
\end{theorem}
The above result is similar to Theorem 3 in \cite{castillo2015bayesian} under classical linear regression setup with i.i.d. observations and Gaussian noise. However, we generalize their result under bandit setup and sub-Gaussian noise by carefully controlling the correlation between noise and observed contexts, which is crucial for our regret analysis.

\subsection{Algorithm  and regret bound}
In this section we introduce the Thomson sampling algorithm for high-dimensional contextual bandit, a pseudo-code for which is provided below in Algorithm \ref{alg: TS}. Similar to the Thompson sampling algorithm in \cite{agrawal2013thompson}, in the $t$th round Algorithm \ref{alg: TS} sets the a specific prior on $\beta$ and updates it sequentially based on the observed rewards and contexts. In particular, it chooses the prior described in \eqref{eq: prior} with an appropriate choice of round-specific prior scaling $\lambda_t$ and updates the posterior using the observed rewards and contexts until $(t-1)$th round. Then a sample is generated from the posterior and an arm $a_t$ is chosen greedily based on the generated sample. 

Now, we show that the Thompson sampling algorithm achieves desirable regret upper bound. 

\begin{theorem}
\label{thm: regret LSB}
Let the Assumption \ref{assumptions: contexts}--\ref{assumptions: noise} hold  with $C = \Theta(\phi_u \vartheta^2 \xi K \log K)$, and $K \geq 2, d\geq T$. Define the quantity $\kappa(\xi, \vartheta, K) := \min \{(4 c_3 K \xi \vartheta^2)^{-1}, 1/2\}$ where $c_3$ is a universal positive constant. 
Also, set the prior scaling $\lambda_t$ as follows: 
\[
(5/3)\overline{\lambda}_t\leq \lambda_t \leq 2 \overline{\lambda}_t, \quad \overline{\lambda}_t = \xmax \sqrt{2t(\log d + \log t)}.
\]
Then there exists a universal constant $C_0>0$ such that we have the following regret bound for Algorithm \ref{alg: TS}:
\begin{align*}
&\bbE\{R(T)\} \lesssim I_b + I_\omega,
\end{align*}

where,
\[
I_b = \left\{\frac{\bmax \xmax \phi_u \vartheta^2 \xi (K \log K)}{ \min \left\{ \kappa^2(\xi, \vartheta, K), \log K  \right\}}\right\} s^* \log(Kd),
\]
\[
  I_\omega = 
  \begin{cases}
        \Phi^{1+\omega} \left( \frac{ s^{* 1+ \omega} (\log d)^{\frac{1+ \omega}{2}} T^{\frac{1 - \omega}{2}}}{\Delta_*^\omega}\right), & \text{for}\; \omega \in [0,1),\\
     
     \Phi^{2} \left(\frac{ s^{* 2} [\log d  + \log T]\log T}{\Delta_*}\right),  & \text{for}\; \omega =1,\\
     
     \frac{\Phi^2 }{(\omega - 1)} \left(\frac{s^{*2 } [\log d + \log T]}{\Delta_*}\right), & \text{for}\; \omega \in (1, \infty)\\
     
\Phi^2 \left( \frac{s^{*2}  [\log d + \log T]}{\Delta_*} \right), & \text{for}\; \omega = \infty,
     
     \end{cases}
\]
 and $\Phi = \sigma \xmax^2   \xi K \left(2 + 40A_4^{-1} +  C_0  K \xi \xmax^2 A_4^{-1}\right)  $ .
\end{theorem}

\paragraph{Discussion on the above result:}

The regret bound provided by Theorem \ref{thm: regret LSB} shows that the regret of the algorithm grows poly-logarithmically in $d$, i.e., $\bbE\{R(T)\} =   O((\log d)^{\frac{1+\omega}{2}})$, when $\omega \in [0,1)$; logarithmically in $d$, i.e., $O(\log d)$ when $\omega \in [1, \infty]$. {Meanwhile, the expected cumulative regret depends polynomially in $T$, i.e., $\bbE\{R(T)\} = O(T^{\frac{1-\omega}{2} })$ when $\omega \in [0,1)$; ploy-logarithmically in $T$; i.e., $\bbE\{R(T)\} = O((\log T)^2)$, when $\omega =1 $.  In $\omega \in (1, \infty]$ regime, the expected cumulative regret depends poly-logarithmically in both the time horizon $T$ and ambient dimension $d$. As $T\ll d$, the expected regret ultimately scales as $O(\log d)$. Comparing our upper bound result with minimax regret lower bound established in Theorem 1 of \cite{li2021regret},  it follows that our algorithm enjoys optimal dependence on both ambient dimension $d$ and time-horizon $T$ when $\omega \in [0,1)$. In $\omega =1$ region, the regret upper bound in the above theorem is optimal up to a $O(\log T)$ term. To the best of our knowledge, there does not exist any result on  minimax lower bound in the regime $\omega > 1$ in the high-dimensional linear contextual bandit literature.} It is worth mentioning that this is an upper bound on the expected (frequentist) regret, as compared to Bayesian regret which is often considered for Thompson sampling based algorithms..

\begin{algorithm}[ht!]
\SetAlgoLined
 Set  $\his_0 = \emptyset$\;
 \For{$t=1,\cdots, T$}{
  \If{$t \leq 1$}{Choose action $a_t$ uniformly over $[K]$\;}
  \Else{
  Set $\overline{\lambda}_t = \xmax \sqrt{2t(\log d + \log t)}$ and choose $\lambda_t \in (5 \overline{\lambda}_t/3, 2 \overline{\lambda}_t)$\; 
  Generate sample $\tilde{\beta}_t\sim {\Pi}(\cdot \mid \cH_{t-1})$ with prior $\Pi$ in \eqref{eq: prior}-\eqref{eq: sparsity-level prior bound}, $\lambda = \lambda_t$ and Gaussian likelihood\;
  Play arm: $a_t = \argmax_{i \in {[K]}} \;x_{i}(t)^\top \tilde{\beta}_t$\;
  }
  Observe reward $r_{a_t}(t)$\;
  Update $\his_{t} \leftarrow \his_{t-1} \cup \{(a_t, r_{a_t}(t), x_{a_t}(t))\}$.
 }
 \caption{Thompson Sampling Algorithm}
 \label{alg: TS}
\end{algorithm}

Intuitively, the initial term $I_b$ in regret upper bound in Theorem \ref{thm: regret LSB} describes the regret caused by the ``burn-in'' period of exploring the space of contexts and it does not contribute to the asymptotic regret growth. Note that we consider running Thompson sampling from the very beginning, without an explicit random exploration phase, in contrast to most of the existing algorithms; the distinction between the burn-in phase and the subsequent phase is only a construct of our theoretical analysis. Furthermore, the constant $\Delta_*$ plays the role of gap parameter which commonly appears in a problem-dependent regret bound \citep{abbasi2011improved}. Note that, for $\omega =0$, we get a problem-independent regret bound of the order $O(s^* \sqrt{T} \log d)$. The appearance of the $\sqrt{T}$, term is not surprising, as the condition $\omega = 0$ poses no prior knowledge on the arm-separability, Thus, in the worst case, the context vectors may fall into each other, making the bandit environment harder to learn. In contrast, as $\omega$ increases the optimal arm becomes more distinguishable than the sub-optimal arms and the bandit environment becomes easier to learn. As a result, the effect of the time horizon becomes less and less severe as $\omega$ increases. In particular, when $\omega \in [1, \infty]$, the time horizon does not affect the asymptotic growth of the regret bound. Finally, as we mainly focus on the case when the number of arms is very small, the quantity $\Phi$ roughly has an inflating effect of $O(1)$ on the regret bound.
\paragraph{Sketch of the proof of Theorem \ref{thm: regret LSB}}
While a self-contained and detailed proof of the above result is given in the Appendix, here we go through the main steps and ideas of the proof. The proof is broadly divided into 3 parts for clarity:
\vspace{-5pt}
\begin{enumerate}[label = (\roman*)]
\itemsep0em
    \item \label{item: empirical sparse eigen} In Section \ref{sec: proof part 1} we will first show that the estimated covariance matrix $\widehat{\Sigma}_t:= X_t^\top X_t/t$ enjoys SRC condition with high probability for sufficiently large $t$. 
    In our analysis, we carefully decouple this complex dependent structure and exploit the special temporal dependence structure of the bandit environment to establish SRC property of $\widehat{\Sigma}_t$.
    
     \item \label{item: empirical restrited eigen} Next, in Section \ref{sec: proof part 2} we will establish a compatibility condition 
     for the matrix $\widehat{\Sigma}_t$. 
     We use a Transfer Lemma (Lemma \ref{lemma: transfer lemma}) which essentially translates the uniform lower bound on $\phi_{\min}(Cs^*, \widehat{\Sigma}_t)$ to a certain compatibility number. 
    
    \item \label{item: Bayesian concentration} Finally, in Section \ref{sec: proof part 3}, under the compatibility condition we use the posterior contraction result in Theorem \ref{thm: informal posterior contraction} to give bound on the per round regret $\Delta_{a_t}(t)$. 
\end{enumerate}
\subsection{Comparison with existing literature}
The stochastic linear bandit problem was first introduced by \cite{auer2002using}, and later was subsequently studied by \cite{dani2008stochastic, chu2011contextual} and many others. Later, \cite{abbasi2011improved} and \cite{abbasi2012online} proposed OFUL algorithm for both low-dimensional and high-dimensional settings. Although there are some similarities in our model parametrization with the setting considered in \cite{abbasi2011improved} and \cite{dani2008stochastic}, there are some significant differences that need attention. To mention a few, the set of contexts considered in \cite{abbasi2011improved, dani2008stochastic} consists of infinitely many feature vectors, whereas in our setting the set of contexts consists of finitely many feature vectors coming from an underlying distribution. Moreover, in \cite{dani2008stochastic}, the set of contexts does not change over time, therefore the optimal 
arm remains the same in every round. In high dimensional bandit, a similar setting is also considered in \cite{HaoLattimore2020, hao2021information}. In contrast, in our setting, due to the randomness of the observed contexts, the optimal arm does not necessarily remain the same in every round. 

Now we shift focus to the regret bound analysis. In low-dimensional setting \cite{rusmevichientong2010linearly} proved a lower bound of $\Omega(d \sqrt{T})$ for both cumulative regret and Bayesian regret in linear bandit setting, where the set of contexts is a compact set of infinitely many feature vectors, e.g., the unit sphere $\bbS^{d-1}\subseteq \bbR^d$. Later, \cite{chu2011contextual} proposed LinUCB algorithm, which has near-optimal regret upper bound $O(\sqrt{T d \log(T \log (T)/ \delta)})$ with probability $1 - \delta$, which corresponds to $\omega = 0$ case. \cite{abbasi2011improved} proved that the expected regret of OFUL algorithm scales as $O(d\sqrt{T})$ in both low-dimensional and high-dimensional setting. In all of these works, the regret bound analysis is based on the worst-case scenario, which leads to polynomial dependence on $d$. In high-dimension, such a polynomial dependence of $d$ may lead to very sub-optimal performance of the aforementioned algorithms. Later, \cite{bastani2020online} proposed LASSO-bandit algorithm, and \cite{wang2018minimax} proposed MCP bandit algorithm which enjoys improved regret upper bound scaling as $O(\log d)$ under the margin condition $\omega =1$, but require forced sampling which could be costly in some practical settings such as medical and marketing applications.
In comparison, our method does not need any forced sampling and does not require the knowledge of the gap parameter $\omega$. On the other hand, our theoretical analysis also covers the regime $\omega \in [0, \infty]$, whereas the results of \cite{bastani2020online} and \cite{wang2018minimax} are only valid for $\omega =1$ case, for which our algorithm enjoys the same $O(\log d)$ dependence in regret bound as LASSO-bandit and MCP-bandit. 

\begin{remark}
  It is worthwhile to point out that the setup in \cite{bastani2020online} and \cite{wang2018minimax} consider different $\beta_{i}^* \in \bbR^d$ for different arms $i \in [K]$ and a single context $x(t)\in \bbR^d$ every round, which is in sharp contrast to our setting. Their formulation can be mathematically reparametrized into our formulation. In particular, for the $i$th arm we construct the new context vectors $\tilde{x}_{i}(t) = (\boldsymbol{0}, \boldsymbol{0},\ldots, x(t)^\top, \boldsymbol{0}, \ldots,\boldsymbol{0})^\top$ where the $i$th block is $x(t)$, thus $x_i(t)\in \bbR^{Kd}$. The common parameter is $\beta^* = (\beta_{1}^{*\top}, \ldots, \beta_{K}^{*\top})^\top$. However, in this case, the new contexts are highly degenerate and violate  Assumption \ref{assumptions: contexts}\ref{item: anti-concentration} and restrict us  from directly applying our result in this case.
\end{remark}

\section{Computation}
\label{sec:computation}
In this section, we discuss the computational challenges and how these are overcome by using Variational Bayes (VB). While priors as \eqref{eq: prior} have been shown to perform well, both empirically and in theory, the discrete model selection component of the prior makes it challenging to allow computation and inference on the posterior. For $\beta\in \bbR^d$, inference using the slab and spike prior requires a combinatorial search over $2^d$ possible models, which in the case of high dimension is computationally infeasible. Fast algorithms are known only in the special diagonal design case and traditional Markov Chain Monte Carlo methods have very slow mixing in such high dimensional cases. Thus, following \cite{ray2021variational} we use Variational Bayes to make computations faster. Specifically, in the sampling step of Algorithm \ref{alg: TS}, we consider the VB approximation of the posterior $\Pi (\cdot \mid \cH_{t-1})$ arising from slab and spike prior with $\Lap(\lambda/\sigma)$ slab in the mean-field family
\begin{align*}
\left\{  \bigotimes_{j=1}^d [\gamma_j \sfN(\mu_j, \sigma_j^2) + (1-\gamma_j)\delta_0]: (\mu_j, \sigma_j, \gamma_j) \in \cR\right\},
\end{align*}
where $\cR = \bbR\times \bbR^+ \times [0,1]$.
We use the \texttt{sparsevb} package \citep{clara2021sparsevb} in R to use the Coordinate Ascent Variational Inference (CAVI) algorithm proposed in \cite{ray2021variational} to obtain the VB posterior. 
This makes the Thompson sampling algorithm much faster as one can efficiently obtain samples from the VB posterior due to its structure. The details of the algorithm for the Variational Bayes Thompson Sampling (VBTS) are in the appendix (see Section \ref{sec: pseudo-code VBTS}).


\section{Numerical Experiment}
\label{sec:numerical}

In both simulations and real data experiments, we present results corresponding to $\lambda_t=1$ for all $t\in [T]$.
Recall that Theorem \ref{thm: regret LSB} suggests that in $t$th round $\lambda_t \asymp \sqrt{t \log d}$ is a reasonable choice for the exact Thompson sampling algorithm. However, in practice, we noticed that such choices of $\lambda_t$ lead to numerical instability. Some recent findings in \cite{ray2021variational} suggest that $\lambda_t$ in the order of $O(\sqrt{t \log d}/s^*)$ should be an appropriate choice, which is smaller than the predicted order of $\lambda_t$ in our main theorem.
Motivated by this, we also present the simulation results for synthetic data experiments 
 with $\lambda_t = \lambda_*\sqrt{t}$ for $\lambda_*\in \{0.2,0.3,0.4,0.5\}$ in the Appendix \ref{sec: lambda simulation}. We found the performance of VBTS to be robust with respect to the choice of the tuning parameter $\lambda_t$. 

\subsection{Synthetic data}
In this section, we illustrate the performance of the VBTS algorithm on a simulated data set. As a benchmark, we consider ESTC \citep{HaoLattimore2020}, LinUCB \citep{abbasi2011improved}, DRLasso \citep{kim2019doubly}, Lasso-L1 confidence ball algorithm \citep{li2021regret}, LinearTS \citep{agrawal2013thompson} and TS algorithm based on Bayesian Lasso \citep{park2008bayesian} (BLasso TS) to compare the performance of VBTS (Algorithm \ref{alg: VBTS}).  

\subsubsection*{Equicorrelated (EC) structure}
We set the number of arms $K = 10$ and we generate the context vectors $\{x_i(t)\}_{i =1}^K$ from multivariate $d$-dimensional Gaussian distribution $\sfN_d(\boldsymbol{0}, \Sigma)$, where $\Sigma_{ij} = \rho^{\abs{i -j} \wedge 1}$ and $\rho = 0.3$. We consider $d = 1000$ and the sparsity $s^* = 5$. We choose the set of active indices $S^*$ uniformly over all the subsets of $[d]$ of size $s^*$. Next, for each choice of $d$, we consider two types generating scheme for $\beta$: 


\vspace{-3mm}
\begin{itemize}
\itemsep0em
    \item \textbf{Setup 1:} $\{U_i\}_{i\in S^*} \overset {i.i.d.}{\sim} \text{Uniform}(0.3, 1)$ and set $\beta_j = U_j(\sum_{\ell \in S^*} U_\ell^2)^{-1/2}\ind(j \in S^*)$.
    \item  \textbf{Setup 2: } $\{Z_i\}_{i\in S^*} \overset {i.i.d.}{\sim} \sfN(0, 1)$ and set $\beta_j = Z_j (\sum_{\ell \in S^*} Z_\ell^2)^{-1/2} \ind(j \in S^*)$.
\end{itemize}
\vspace{-2mm}

We run 40 independent simulations and plot the mean cumulative regret with 95\% confidence band in Figure \ref{fig: EC, setup=1}-\ref{fig: EC, setup=2}. In all the setups, we see that VBTS outperforms its competitors by a wide margin. VBTS also enjoys superior empirical performance under the autoregressive (AR) model (see Figure \ref{fig: AR, setup=1}-\ref{fig: AR, setup=2}) with auto-correlation coefficient $0.3$ and the details of the simulations can be found in Appendix \ref{sec: simulation details-AR(1)}. 
Table \ref{tab: time table} shows the mean execution time (across Setup 1 and 2) of all TS algorithms. Among the class of TS algorithms, VBTS outperforms its other competing algorithms.

\begin{figure*}
\centering
    \begin{subfigure}{0.48\linewidth}
        \includegraphics[width=\columnwidth]{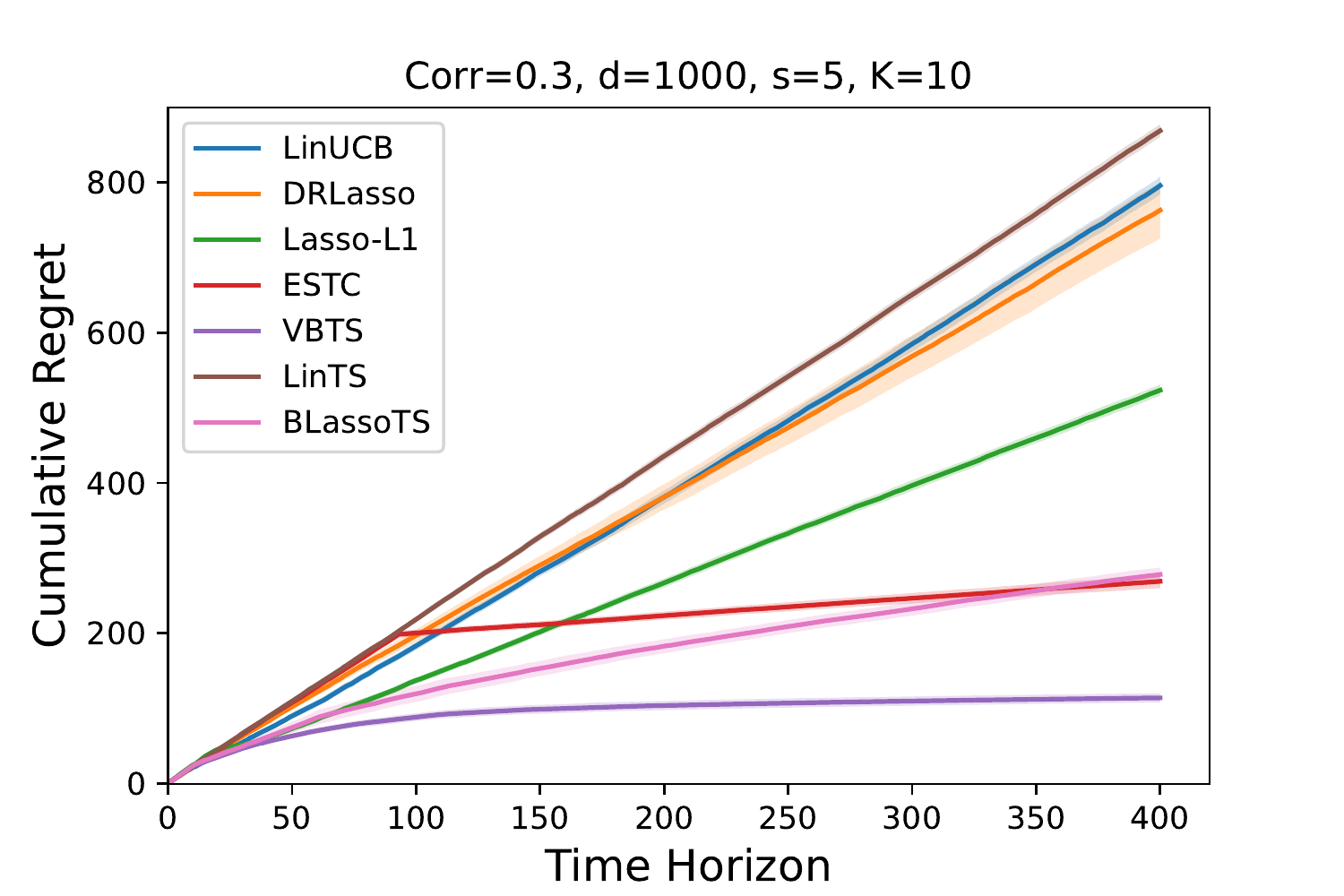}
        \caption{EC (Setup 1)}
        \label{fig: EC, setup=1}
    \end{subfigure}\hfill   
    \begin{subfigure}{0.48\linewidth}
        \includegraphics[width=\columnwidth]{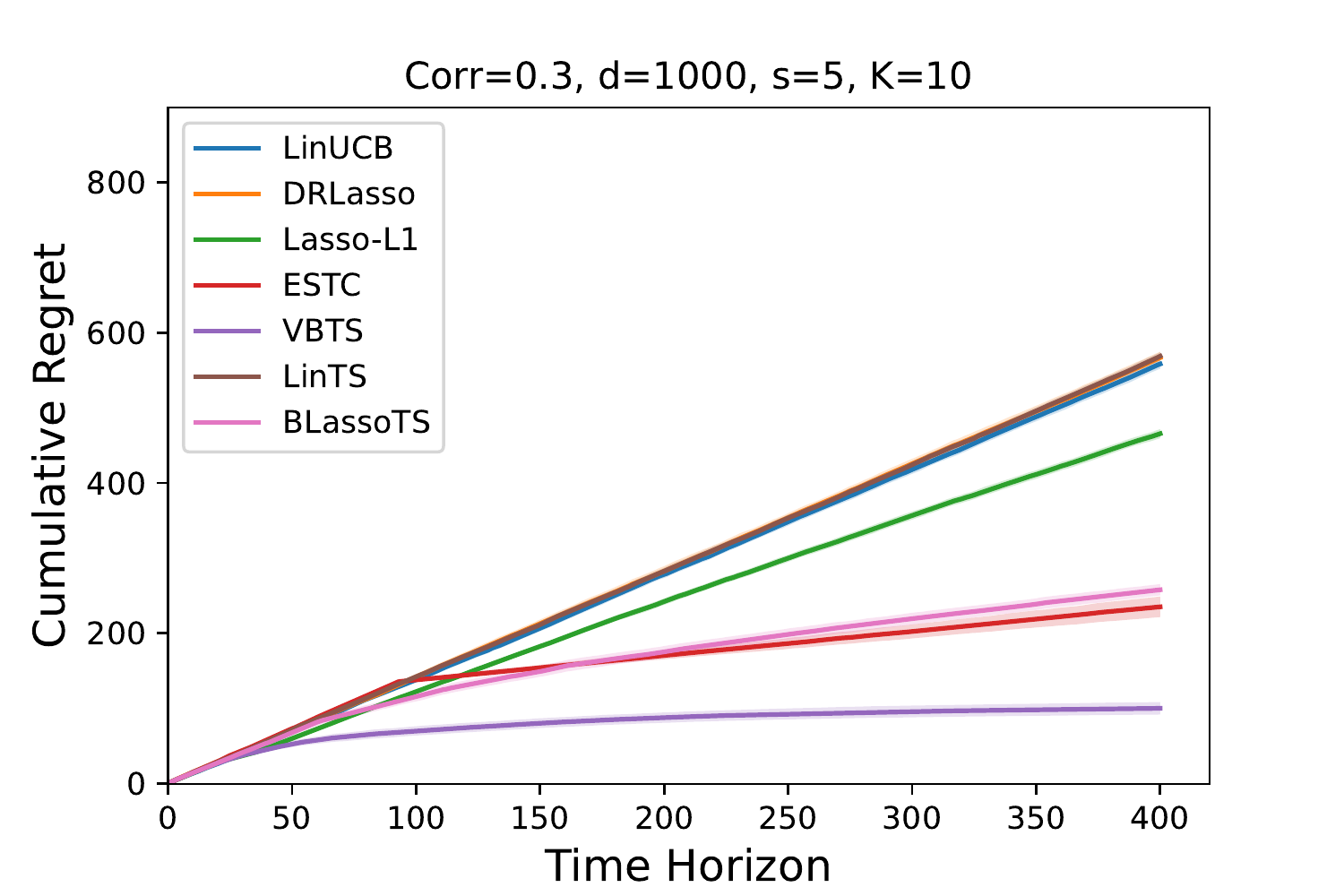}
        \caption{EC (Setup 2)}
        \label{fig: EC, setup=2}
    \end{subfigure}\hfill
    \begin{subfigure}{0.48\linewidth}
        \includegraphics[width=\columnwidth]{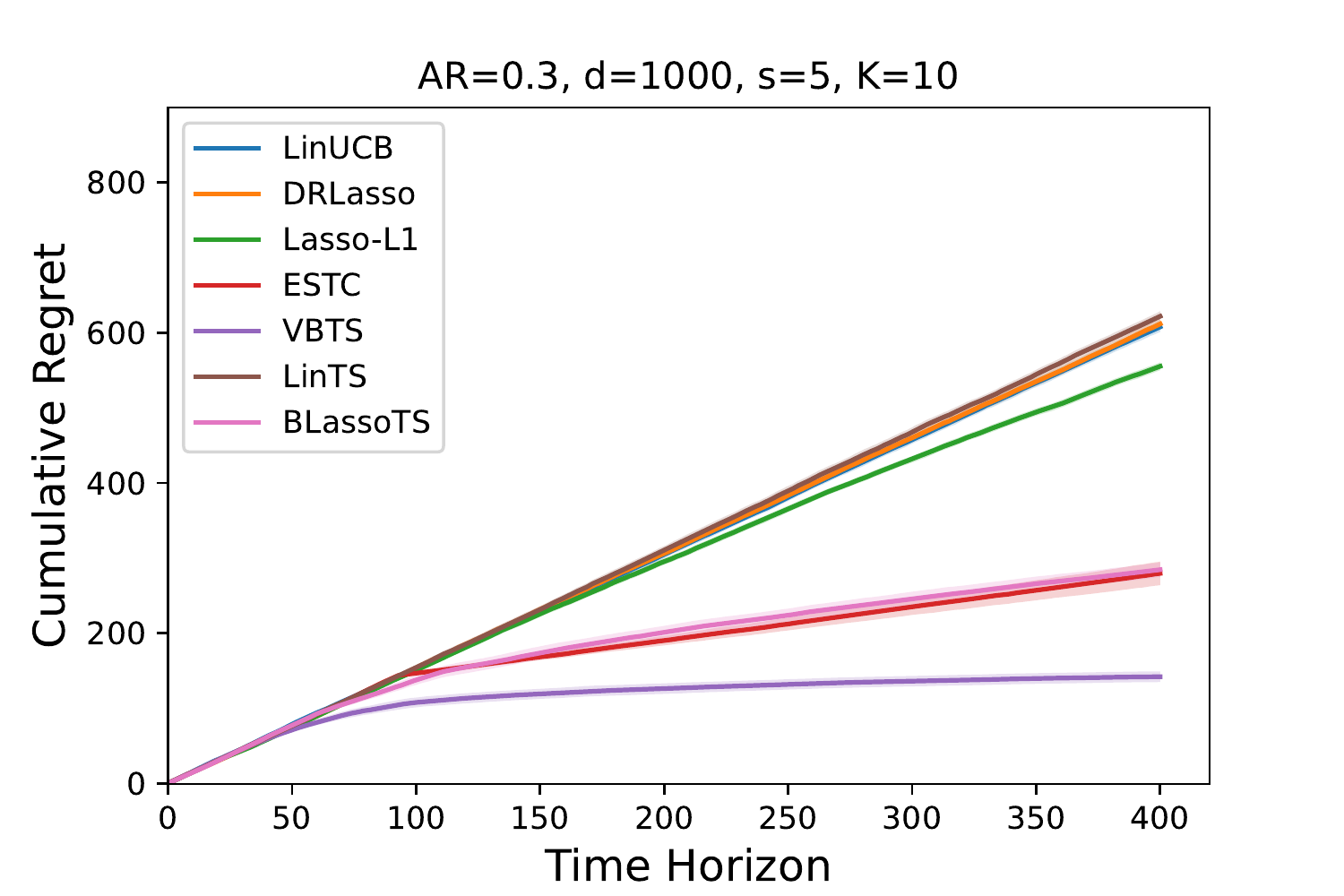}
            \caption{AR (Setup 1)}
        \label{fig: AR, setup=1}
    \end{subfigure}\hfill
    \begin{subfigure}{0.48\linewidth}
        \includegraphics[width=\columnwidth]{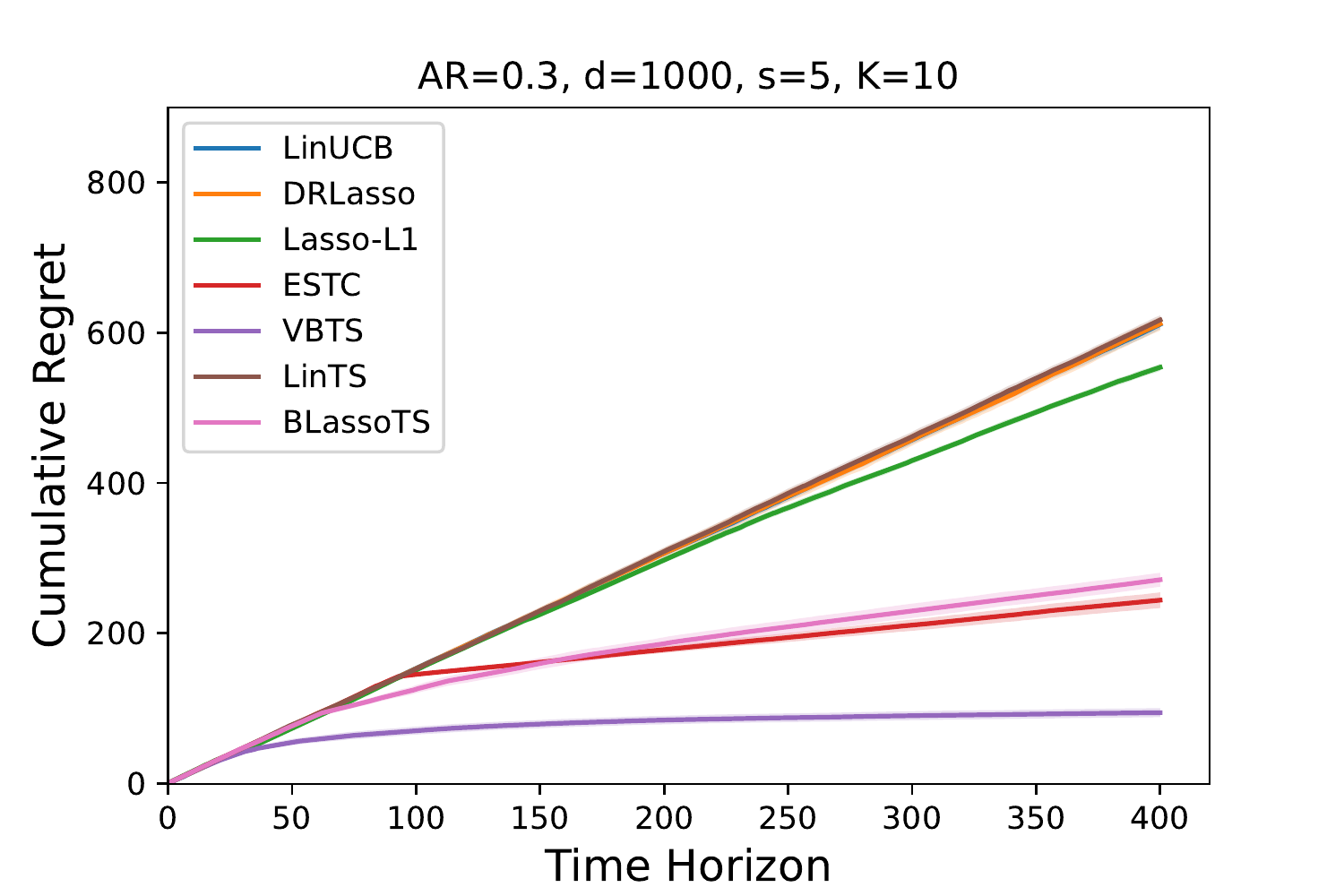}
        \caption{AR (Setup 2)}
        \label{fig: AR, setup=2}
    \end{subfigure}
\vspace{-3mm}    
\caption{Cumulative regret of competing algorithms.}
    \label{fig: cumulative regret plots}
    \end{figure*}


\renewcommand{\arraystretch}{1.5}
\setlength{\tabcolsep}{1pt}
\begin{table}[h]
\small
    \centering
    \caption{Time comparison among the competing algorithms.}
    \vspace{0.1in}
    \begin{tabular}{ |c|c|c|c| }
    \hline
    \multirow{2}{*}{Type} & \multirow{2}{*}{Algorithm} & \multicolumn{2}{|c|}{Mean time of execution (seconds)} \\
    \cline{3-4}
    & & Equicorrelated & Auto-regressive \\
    \hline
    \hline
    \multirow{3}{5em}{\hspace{15pt}TS} & LinTS & 1344.39 & 1346.46\\ 
    & BLasso TS & 1511.68 & 1455.53\\
    & \textbf{VBTS} & \textbf{29.33} & \textbf{27.65} \\ 
    \hline
    \end{tabular}
    
    \label{tab: time table}
\end{table}

\subsection{Real data - \texttt{gravier} Breast Carcinoma Data}

We consider breast cancer data \texttt{gravier} (\texttt{microarray} package in \texttt{R}) for 168 patients to predict metastasis of  breast carcinoma based on 2905 gene expressions (bacterial artificial chromosome or BAC array). The goal of the learner is to identify the positive cases.


Similar to \cite{kuzborskij2019efficient, chen2021efficient}, 
in our experimental setup, we convert the breast cancer classification problem into 2-armed contextual bandit problem.  More details about the data and reward generation process are provided in Appendix \ref{sec: details real-data}.
We perform 10 independent Monte Carlo simulations and plot the expected regret of VBTS  in Figure \ref{fig: cumaltive regret breast cancer} along with its competitors. In this experiment, we omit LinUCB and LinTS algorithms as they were performing far worse compared to the existing ones in Figure \ref{fig: cumaltive regret breast cancer}. The figure shows that VBTS and Lasso-L1 confidence ball algorithms are by far the clear winners in terms of cumulative regret. However, upon closer look, we see that VBTS is slightly better than Lasso-L1 confidence in terms of cumulative regret. In terms of accuracy, Lasso-L1 and VBTS are in the same ball park as seen in Table \ref{tab: spam accuracy}. 

\begin{figure}[h!]
\centering
\includegraphics[width=.8\columnwidth]{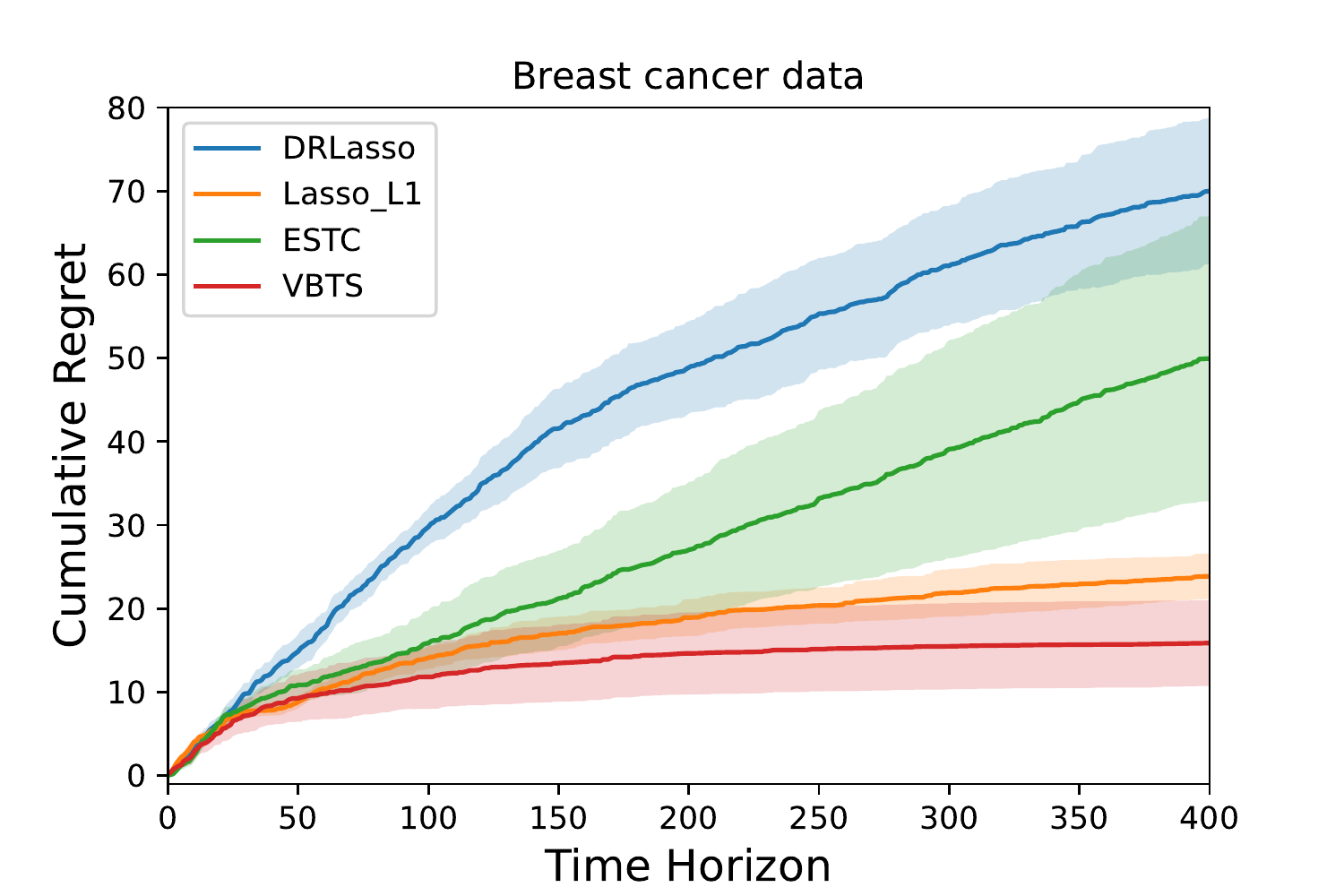}
\vspace{-3mm}
\caption{Cumulative regret plot for breast cancer data set.}
    \label{fig: cumaltive regret breast cancer}
    \end{figure}

\vspace{-3mm}
\renewcommand{\arraystretch}{1.5}
\setlength{\tabcolsep}{3pt}
\begin{table}[h]
\small
    \centering
    \caption{Classification accuracy of competing algorithms.}
    \vspace{0.1in}
\begin{tabular}{ |c|c|c|c|c| } 
 \hline
 Algorithm & DRLasso & Lasso-L1 & ESTC & \bf{VBTS}  \\ 
 \hline
 Accuracy(\%) & 65.63   & 81.20   & 73.32 & \textbf{81.88}\\
 \hline
\end{tabular}
\label{tab: spam accuracy}
\end{table}

\section{Conclusion}
In this paper, we consider the stochastic linear contextual bandit problem with high-dimensional sparse features and a fixed number of arms. We propose a Thompson sampling algorithm for this problem by placing a suitable \emph{sparsity-inducing} prior on the unknown parameter to induce sparsity. We also develop a crucial posterior contraction result for \textit{non-i.i.d.}\ data that allows us to obtain an almost \emph{dimension independent} regret bound for our proposed algorithm. We explicitly point out the dependences on $d$ and $T$ for different arm-separation regimes parameterized by $\omega$, which is also minimax optimal for $\omega\in [0,1)$. Moreover, the choice of prior allows us to devise a Variational Bayes algorithm that enjoys computational expediency over traditional MCMC. We demonstrate the superior performance of our algorithm through extensive simulation studies. We finally perform an experiment on a \texttt{gravier} dataset by converting the classification problem into a 2-armed contextual bandit problem, for which our method performs better compared to other existing algorithms.

Now we point the readers toward some of the natural research directions that we plan to cover in our future works. 
In terms of regret analysis, similar to most of the recent works in high dimensional contextual bandits, relies on upper bounding the regret through estimation of the parameter, i.e., we rely on the estimation of $\beta^*$ to be able to provide meaningful regret bound. However, this should not be required - as an example, consider the case where the first coordinate of $x_{i}(t)$ is 0 for all $i\in [K],t\in [T]$. Then the first coordinate of $\beta^*$ is not estimable, however, this does not pose any problem to designing a sensible policy since this coordinate does not appear in the regret. Unfortunately, Assumption \ref{assumptions: contexts}\ref{item: anti-concentration} is not satisfied for such degeneracy in the contexts and as a result, it would require a modified analysis of the regret bound.
 Secondly, we underscore the fact that in our setup we adopt the Variational Bayes framework only to sidestep the computational hurdles of MCMC arising from a myriad of challenges such as   slow mixing times of the chains,  lack of easy implementation, etc. However, in high-dimensional regression setup \cite{yang2016computational} has proposed Metropolis-Hastings algorithms based on truncated sparsity priors that do not meet the above roadblocks. It could be very well possible that some other prior structure will allow us to design more efficient MCMC algorithms with faster mixing times in the high-dimensional SLCB setup along with theoretical guarantees.
Finally, we also plan to analyze Thompson sampling for high-dimensional generalized contextual bandit problems.

\paragraph{Author Contribution:}
All authors conceived and carried out the research project
jointly. S.C. and S.R. jointly wrote the paper and code for numerical experiments. A.T. helped edit the paper.
\bibliographystyle{apalike}  
\bibliography{references} 

\begin{appendices}
\section{Details of Simulations}
\subsection{Simulation details for AR(1) structure}
\label{sec: simulation details-AR(1)}
We set the number of arms $K = 10$ and we generate the context vectors $\{x_i(t)\}_{i =1}^K$ from multivariate $d$-dimensional Gaussian distribution $\sfN_d(\boldsymbol{0}, \Sigma)$, where $\Sigma_{ij} = \phi^{\abs{i -j} }$ and $\phi = 0.3$. We consider $d = 1000$ and the sparsity $s^* = 5$. We choose the set of active indices $\calS^*$ uniformly over all the subsets of $[d]$ of size $s^*$. Next, for each choice of $d$, we consider two types generating scheme for $\beta$:
\begin{itemize}
    
    \item Setup 1: $\{U_i\}_{i\in \calS^*} \overset {i.i.d.}{\sim} \text{Uniform}(0.3, 1)$ and set $\beta$ as the following:
    \[
    \beta_j = \begin{cases}
  \frac{U_j}{\sqrt{\sum_{\ell \in \calS^*} U_\ell^2}},  &  \text{ if $ j \in \calS^*$,} \\
  0, &  \text{ otherwise.}
\end{cases}.
\]
    \item  Setup 2: $\{Z_i\}_{i\in \calS^*} \overset {i.i.d.}{\sim} \text{Normal}(0, 1)$ and set $\beta$ as the following:
    \[
    \beta_j = \begin{cases}
  \frac{Z_j}{\sqrt{\sum_{\ell \in \calS^*} Z_\ell^2}},  &  \text{ if $ j \in \calS^*$,} \\
  0, &  \text{ otherwise.}
\end{cases}.
\]
\end{itemize}
We run 40 independent simulations and plot the mean cumulative regret with 95\% confidence band in Figure \ref{fig: cumulative regret plots}. In all the setups, we see that VBTS outperforms its competitors by a wide margin. Similar to the previous simulation example, in this case also Table \ref{tab: time table} shows that VBTS is far better in terms of mean execution time than its competitors in the class of TS algorithms. 

\subsection{Siumlation for different choices of $\lambda$}
\label{sec: lambda simulation}
As discussed in the first paragraph of Section \ref{sec:numerical}, for each of these simulation settings, we tried a few choices for the tuning parameter $\lambda_t$. In addition to the default choice of $\lambda_t=1$ (for all time points $t$), we also explored the performance of the algorithm under growing $\lambda$, as required by our theoretical results. In particular, we tried $\lambda_t = \lambda_*\sqrt{t}$ for $\lambda_*\in\{0.2, 0.3, 0.4, 0.5\}$. For comparison, we only kept the faster optimism based methods DRLasso, Lasso-L1 and ESTC. We found the results to be roughly robust to the choice of this tuning parameter. The results are summarized in Figure \ref{fig: cumaltive regret equi lambda} and Figure \ref{fig: cumaltive regret AR lambda} below. However, we found that larger values of $\lambda_*$ lead to numerical issues, we conjecture that this is an artifact of the variational Bayes approximation, rather than the prior itself. For our simulation settings, the choice $\sqrt{t\log d}/s^* \approx 0.5 \sqrt{t}$ and hence by the findings in \cite{ray2021variational}, values of $\lambda_*$ higher than this may yield inaccurate Variational Bayes estimation. 

    


    

\begin{figure}[h!]
\centering
    \begin{subfigure}{0.47\linewidth}
	\includegraphics[width=\linewidth]{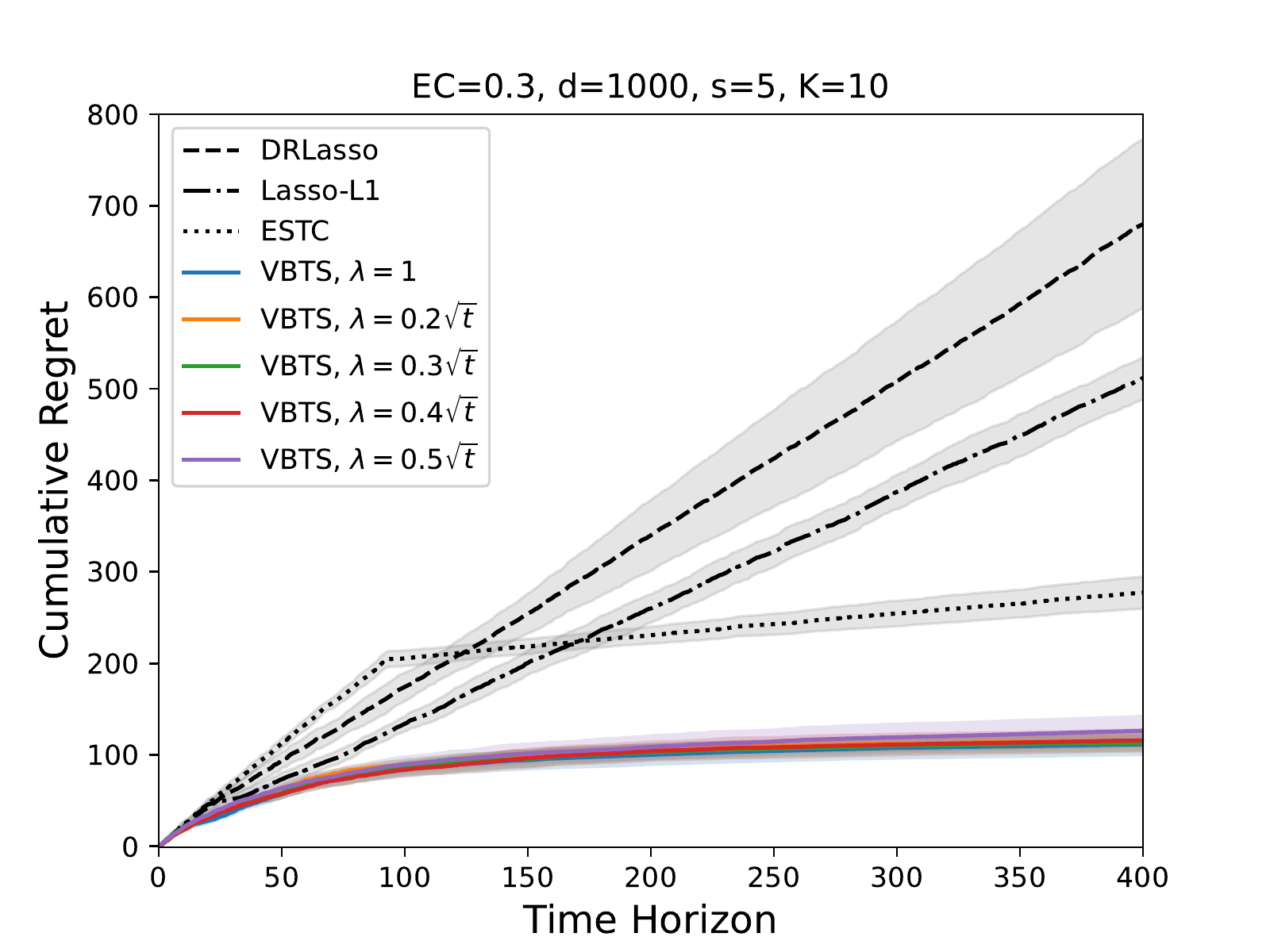}
    \caption{Setup 1}
    \end{subfigure}\hfill
    \begin{subfigure}{0.47\linewidth}
\includegraphics[width=\linewidth]{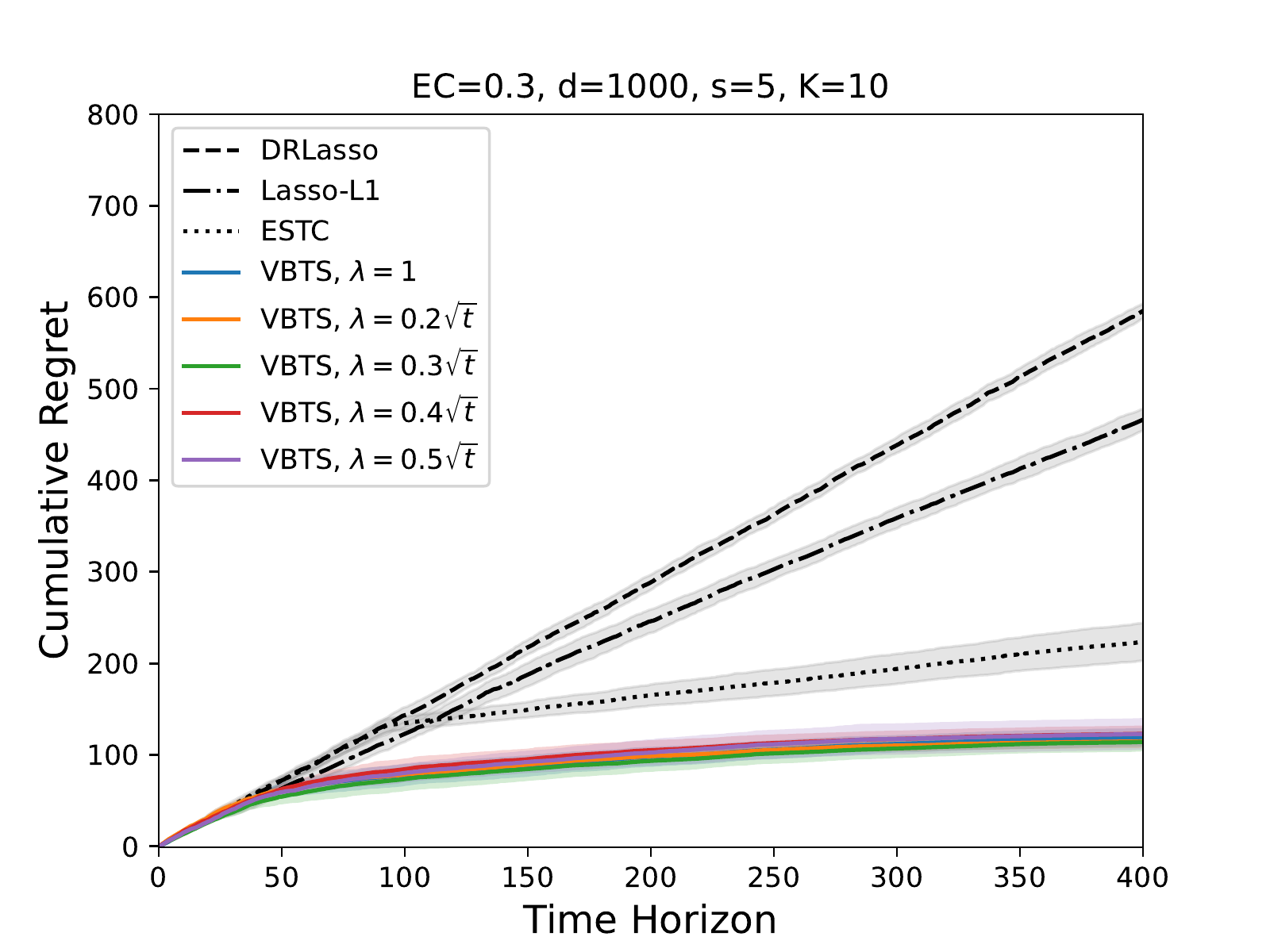}
    \caption{Setup 2}
    \end{subfigure}
\caption{Regret bound for equi-correlated design for different tuning parameter choices}
    \label{fig: cumaltive regret equi lambda}
    \end{figure}

\begin{figure}[h!]
\centering
    \begin{subfigure}{0.47\linewidth}
\includegraphics[width=\linewidth]{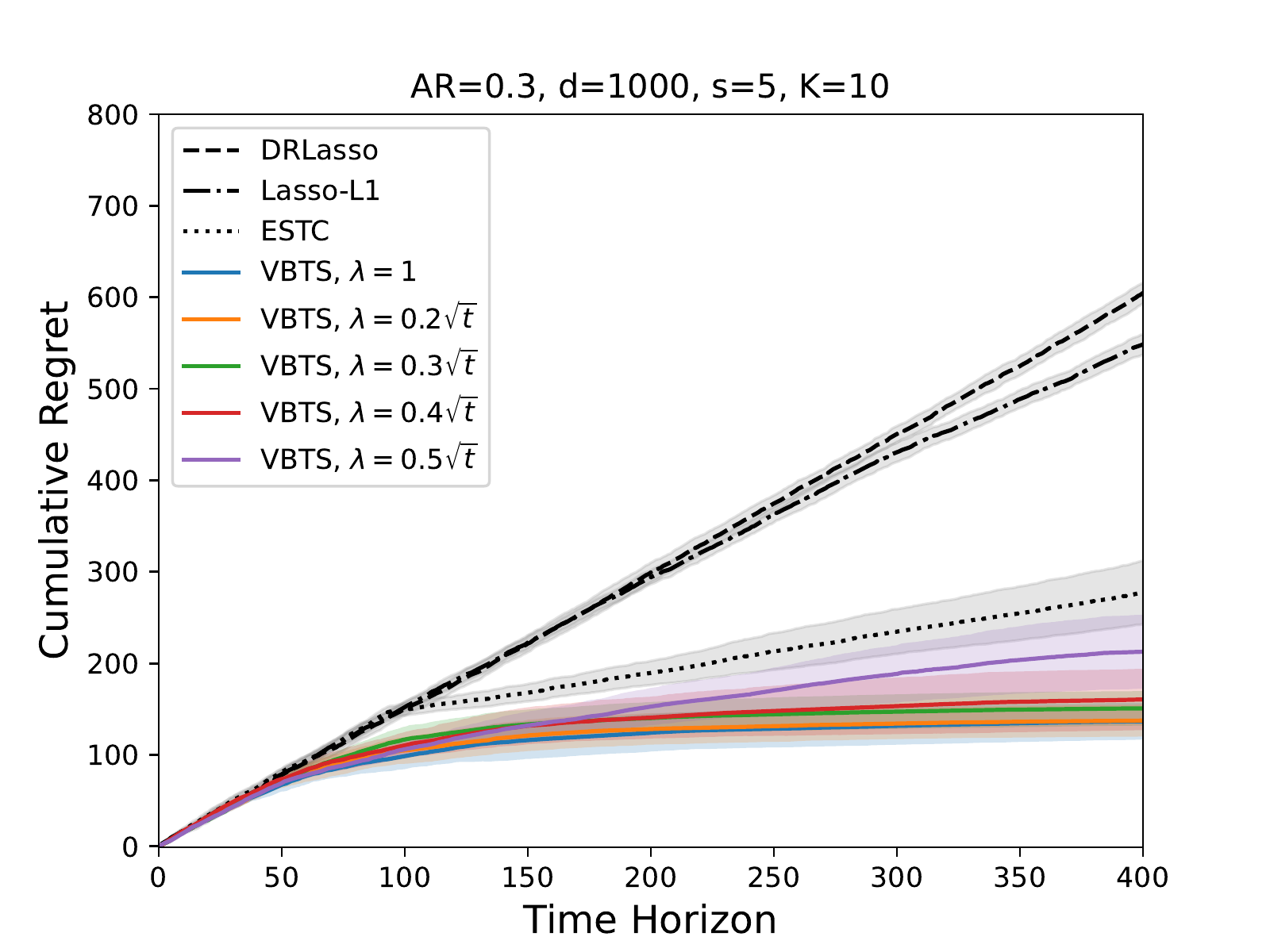}
    \caption{Setup 1}
    \end{subfigure}\hfill
    \begin{subfigure}{0.47\linewidth}
\includegraphics[width=\linewidth]{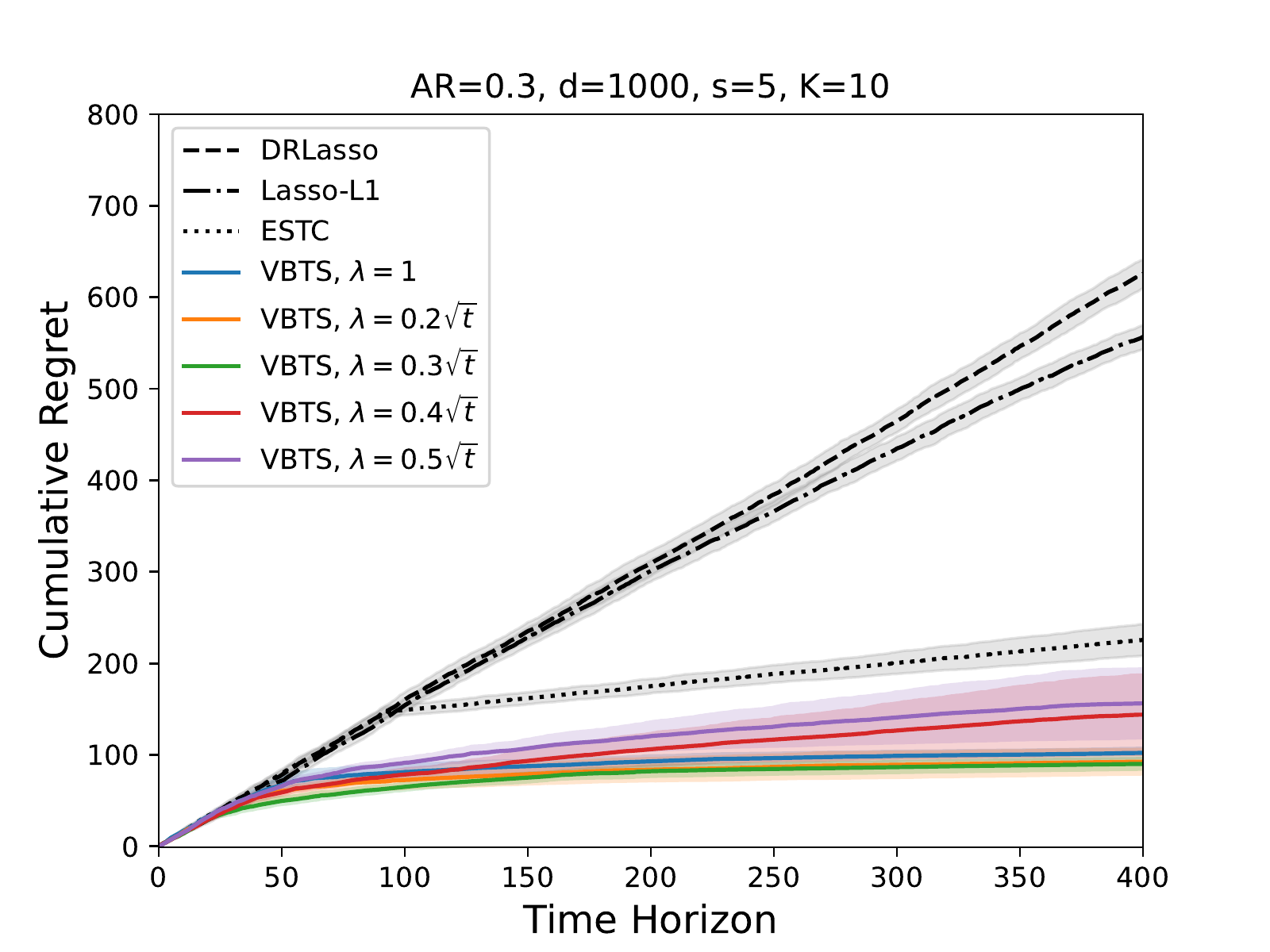}
    \caption{Setup 2}
    \end{subfigure}
\caption{Regret bound for auto-regressive design for different tuning parameter choices}
    \label{fig: cumaltive regret AR lambda}
    \end{figure}

\subsection{Details of real data experiment}
\label{sec: details real-data}
We consider breast cancer data \texttt{gravier} (\texttt{microarray} package in \texttt{R}) for 168 patients to predict metastasis of  breast carcinoma based on 2905 gene expressions (bacterial artificial chromosome or BAC array). \citep{gravier2010prognostic} considered small, invasive ductal carcinomas without axillary lymph node involvement (T1T2N0) to predict metastasis of small node-negative breast carcinoma. Using comparative genomic hybridization arrays, they examined 168 patients over a five-year period. The 111 patients with no event after diagnosis were labeled good (class 0), and the 57 patients with early metastasis were labeled poor (class 1). The 2905 gene expression levels were normalized with a $\log_2$ transformation.

Similar to \cite{kuzborskij2019efficient, chen2021efficient}, 
in our experimental setup we convert the breast cancer classification problem into 2-armed contextual bandit problem as follows: Given the \texttt{gravier} data set with 2 classes, we first set Class 1 as the target class. In each round, the environment randomly draws one sample from each class and composes a set of contexts of 2 samples. The learner chooses one sample and observes the reward following a logit model. In particular, we model the reward as 
\[
r(t):= \log \left\{\frac{\pr(\text{Selected class = 1})}{\pr(\text{Selected class = 0})}\right\} = x_{a_t}^\top \beta^* + \epsilon(t),
\]
where $a_t \in \{1,2\}$ is the selected arm at round $t$. Thus, small cumulative regret insinuates that the learner is able to differentiate the positive patients eventually. Such concepts can be used for constructing online classifiers to differentiate carcinoma metastasis from healthy patients based on gene expression data.  However, in practice, we can not measure the regret defined in \eqref{eq: regret definition}, unless we have the knowledge of $\beta^*$. To resolve this issue, we first fit a logit model on the whole \texttt{gravier} data set and consider the estimated $\hat{\beta}$  as the ground truth and report the expected regret with respect to the estimated $\beta$. As reported in \cite{gravier2010prognostic}, 24 (out of 2905) BACs showed statistically significant difference (comparing Cy3/Cy5 values) between the two groups, motivating the use of a sparse logit model in our case. The estimated $\beta^*$ in our sparse logistic model on the dataset had a sparsity of 18 using the dataset. In addition to this, we also treat the estimated noise variance from the fitted logit model as the true noise variance of the error induced by the environment in each round. 

\section{Proof of Theorem \ref{thm: regret LSB}}
\label{appendix: regret bound}
In this section, we present the detailed proof of Theorem \ref{thm: regret LSB}. First, for clarity of presentation, we introduce some notations. We use $X_{t}$ to denote the matrix $(x_{a_1}(1), \ldots, x_{a_t}(t))^\top \in \bbR^{t \times d}$. Given this, we denote the covariance matrix $\widehat{\Sigma}_t = X_{t}^\top X_{t}/t$. Next, we define the set 
$$\bbS_0^{d-1}(s)\triangleq \bbS^{d-1} \cap \{v \; : \; \norm{v}_0 \leq s\}.$$
We also define the following:
\begin{definition}
\label{def: restricted cone}
For a index set $I \subseteq [d]$ and $\alpha \in \bbR^+$, we define the restricted cone as
\[
\bbC_\alpha(I):= \{v \in \bbR^d\;: \; \norm{v_{I^c}}_1 \leq \alpha \norm{v_I}_1 , v_I \neq 0 \}
\].
\end{definition}
In high-dimensional literature one typically assumes compatibility condition on the design matrix $X$, i.e.,
\begin{equation}
    \label{eq: restricted eigen value cond}
    {\phi}_{\comp}(S^*;X):= \inf_{\delta \in \bbC_7(S^*)}  \frac{\norm{X \delta}_2 \abs{S^*}^{1/2}}{t^{1/2} \norm{\delta}_1} >0,
\end{equation}
where $S^* = \{j : \beta^*_j \neq 0\}$. This is mainly to guarantee the estimation accuracy of high-dimensional estimators like LASSO \citep{bickel2009simultaneous} or to show the posterior consistency in Bayesian high dimensional literature \citep{castillo2015bayesian}.

As discussed in the main paper, we prove the theorem in three parts, the subsequent sections deal with each part separately.

\subsection{Proof of part \ref{item: empirical sparse eigen}}
\label{sec: proof part 1}
In this section we will show that the matrix $\widehat{\Sigma}_t$ enjoys the SRC condition with high probability.
As a warm up, we recall the definition of Orlicz norms:

\begin{definition}[Orlicz norms]
\label{def: orlicz norm}
For random variable $Z$ we have the followings:
\begin{enumerate}[label = (\alph*)]
    \item \label{item: psi2} The \emph{sub-Gaussian} norm of a random variable $Z$, denoted $\norm{Z}_{\psi_2}$, is defined as 
    \[
    \norm{Z}_{\psi_2}:= \inf\{\lambda>0 : \bbE\{\exp(Z^2/\lambda^2)\}\leq 2 \}. 
    \]
    
    \item \label{item: psi1} The \emph{sub-Exponential} norm of a random variable $Z$, denoted $\norm{Z}_{\psi_1}$, is defined as 
    \[
    \norm{Z}_{\psi_1}:= \inf\{\lambda>0 : \bbE\{\exp(\abs{Z}/\lambda)\} \leq 2 \}. 
    \]
\end{enumerate}
\end{definition}
The details and related properties can be found in Section 2.5.2 and Section 2.7 in \cite{vershynin2018high}. The following is a relationship between sub-gaussian and sub-exponential random variables.

\begin{lemma}[Sub-Exponential and sub-Gaussian squared]
\label{lemma: sub-exp and sub-gauss}
A random variable $Z$ is sub-Gaussian iff $Z^2$ is sub-Exponential. Moreover,
\[
\norm{Z^2}_{\psi_1} = \norm{Z}_{\psi_2}^2.
\]

\end{lemma}
\begin{proof}
 The proof can be found in Lemma 2.7.4 of \cite{vershynin2018high}
\end{proof}

\begin{lemma}[Bernstein's inequality]
\label{lemma: bernstein}
Let $Z_1, \ldots, Z_N$ be independent mean-zero sub-exponential random variable. Then for every $\delta \geq 0$ we have
\[
\pr \left( \abs{\frac{1}{N} \sum_{i = 1}^N Z_i} \geq \delta \right) \leq 2 \exp \left\{ -c_2  \min \left( \frac{\delta^2}{K_0^2}, \frac{\delta}{K_0}\right) N\right\},
\]
where $K_0 = \max_{i \in [N]} \norm{Z_i}_{\psi_1}$.
\end{lemma}

\begin{proof}
 The proof can be found in Corollary 2.8.3 of \cite{vershynin2018high}.
\end{proof}

\begin{proposition}[Empirical SRC]
\label{prop: empirical SRC}
Let $\varepsilon = \min\{1/4,  1/(\tilde{c} \phi_u \vartheta^2 \xi K \log K + 3)\}$ for some universal constant $\tilde{c}>0$ and also define the quantity $\kappa(\xi, \vartheta, K) \triangleq \min \{(4 c_3 K \xi \vartheta^2)^{-1}, 1/2\}$ for some universal positive constants $c_3$. Then the followings are true for any constant $C>0$:
\label{prop: sparse eigenvalue empirical}
\begin{equation*}
    \pr\left( \phi_{\max} (Cs^*; \widehat{\Sigma}_t) \geq \frac{c_9 \vartheta^2 \phi_u \log K}{1 - 3 \varepsilon}\right) \leq \exp\left\{ - c_8 t \log K + C s^* \log K + Cs^* \log(3d/\varepsilon)\right\},
\end{equation*}

\begin{equation*}
\begin{aligned}
 \pr \left( \phi_{\min}(C s^*; \widehat{\Sigma}_t)\leq \frac{1}{8K \xi} \right)
 \leq  & \; 2 \exp\{- c_2 \kappa^2(\xi, \vartheta, K)t + C s^* \log K + Cs^* \log(3d/\varepsilon)\}\\
& + \exp\left\{ - c_8  \log (K) t + C s^* \log K + Cs^* \log(3d/\varepsilon)\right\},
\end{aligned}
\end{equation*}
where all $c_j$'s in the above display are universal positive constants.
\end{proposition}

\begin{proof}
We will first show the SRC condition for a fixed vector $v \in \bbS_0^{d-1}(C s^*)$. Then, the whole argument will be extended via a $\varepsilon$-net argument.
\\
\textbf{Analysis for a fixed vector $v$:}
Let $v \in \bbS_0^{d-1}(Cs^*)$ be a fixed vector. Now note that following fact:
 \begin{equation*}
     v^\top \widehat{\Sigma}_t v = \frac{1}{t} \sum_{\tau=1}^t \{v^\top x_{a_\tau}(\tau) \}^2  \geq \frac{1}{t}\sum_{\tau =1}^t \min_{i \in [K]}\{v^\top x_{i}(\tau)\}^2.
 \end{equation*}
 We define $Z_{\tau,v} \triangleq  \min_{i \in [K]}\{v^\top x_{i}(\tau)\}^2$ and note that for a fixed $v\in \bbS_0^{d-1}(C s^*)$, the random variables $\{Z_{\tau, v}\}_{\tau =1}^t$ are i.i.d. across the time points. Moreover, due to Assumption \ref{assumptions: contexts}\ref{item: context_subgaussian} and Lemma \ref{lemma: sub-exp and sub-gauss}, we have
 \[
 \norm{Z_{\tau, v}}_{\psi_1} = \norm{\min_{i \in [K]} \abs{v^\top x_i(\tau)}}_{\psi_2}^2 \leq c_3 \vartheta^2.
 \]
 
 Thus, $\{Z_{\tau, v}\}_{\tau = 1}^t$ are i.i.d sub-exponential random variables. First we will show that $\bbE(Z_{\tau,v})$ is uniformly lower bounded. Due to Assumption \ref{assumptions: contexts}\ref{item: anti-concentration} we have 
 \[
 \pr(Z_{\tau , v} \leq h) \leq \sum_{i = 1}^K \pr \{(v^\top x_i(\tau))^2 \leq h\} \leq K \xi h.
 \]
 Thus we have the following:
\begin{equation}
\label{eq: expectation lower bound}
\begin{aligned}
    \bbE(Z_{\tau, v}) &= \int_{0}^\infty \pr(Z_{\tau, v} \geq u)\; du\\
    & \geq \int_{0}^h \pr(Z_{\tau, v} \geq u)\; du\\
    & \geq \int_{0}^h (1 - K \xi u)\; du\\
    & = h(1 - K \xi h/2).
\end{aligned}
\end{equation}
Setting $h = 1/(K \xi)$ in Equation \eqref{eq: expectation lower bound} yields $\bbE(Z_{\tau,v}) \geq \frac{1}{2 K \xi}$.
 Now, using Lemma \ref{lemma: bernstein}, we have the following for a $\mu \in (0, c_1/ \norm{Z_{1,v}}_{\psi_1})$ and $\delta >0$:
 
 \begin{equation*}
 \begin{aligned}
    \pr \left(\frac{1}{t}\sum_{\tau =1}^t \{ Z_{\tau, v} - \bbE(Z_{\tau, v})\} \geq \delta \right) &\leq 2 \exp \left\{ - c_2 \min \left( \frac{\delta^2}{\norm{Z_{1,v}}_{\psi_1}^2}, \frac{\delta}{\norm{Z_{1,v}}_{\psi_1}} \right) t \right\}\\
    &  \leq 2 \exp \left\{ - c_2 \min \left( \frac{\delta^2}{c_3^2 \vartheta^4}, \frac{\delta}{c_3 \vartheta^2} \right) t \right\}  .
    \end{aligned}
 \end{equation*}
 

Now choose $\delta = \min \{( 4K \xi)^{-1}, c_3 \vartheta^2/2\}$ to finally get 
\begin{equation}
\label{eq: sparse eigen concentration 1}
 \begin{aligned}
    \pr \left(\abs{\frac{1}{t}\sum_{\tau =1}^t \{ Z_{\tau, v} - \bbE(Z_{\tau, v})\}} \geq \frac{1}{4 K \xi} \right)  \leq 2 \exp \left\{ - c_2 \kappa^2(\xi, \vartheta, K) t \right\}  ,
    \end{aligned}
 \end{equation}
 where $\kappa(\xi, \vartheta, K) = \min \{(4 c_3 K \xi \vartheta^2)^{-1}, 1/2\}$. Now recall that $\bbE(Z_{\tau,v}) \geq 1/ (2 K \xi)$, which shows that
 \begin{equation}
 \label{eq: sparse eigen concentration 2}
     \pr \left\{ \frac{1}{t} \sum_{\tau =1}^t Z_{\tau,v} \geq \frac{1}{ 4 K \xi} \right\} \leq  2 \exp \left\{ - c_2 \kappa^2(\xi, \vartheta, K) t \right\}, \quad \forall \; v \in \bbS_0^{d-1}(C s^*).
 \end{equation}
\\
\textbf{$\varepsilon$-net argument:} 
We consider a $\varepsilon$-net of the space $\bbS_0^{d-1}(C s^*)$ constructed in a specific way which will be described shortly. We denote it by $\cN_\varepsilon$. Let $J \subseteq [d]$ such that $\abs{J} = C s^*$ and consider the set $E_J = \bbS^{d-1} \cap \lspan\{ e_j \; : \; j \in J\}$. Here $e_j$ denotes the $j$th canonical basis of $\bbR^d$. Thus we have \[\bbS_0^{d-1}(C s^*) = \bigcup_{J : \abs{J} = C s^*}E_J.\] Now we describe the procedure of constructing a net for $\bbS_0^{d-1}(Cs^*)$ which is essential for controlling the parse eigenvalues.
\newline
\underline{Greedy construction of net:}
\begin{itemize}
    \item  Construct a $\varepsilon$-net of $E_J$ for each $J$ of size $C s^*$. We denote this net by $\cN_{\varepsilon, J}$. Note that $\abs{\cN_{\varepsilon, J}} \leq (3/ \varepsilon)^{C s^*}$ \citep[Corollary 4.2.13]{vershynin2018high} for $\varepsilon \in (0,1)$ as $E_J$ can be viewed as an unit ball embedded in $\bbR^{C s^*}$.
    
    \item Then the net $\cN_\varepsilon$ is constructed by taking union over all the $\cN_{\varepsilon, J}$, i.e., 
    \[
    \cN_\varepsilon = \bigcup_{J : \abs{J} = C s^*} \cN_{\varepsilon, J} .
    \]
\end{itemize}

Thus, from from the construction we have

\begin{equation}
\label{eq: covering number}
    \abs{\cN_\varepsilon} \leq \binom{d}{Cs^*}\left( \frac{3}{\varepsilon}\right)^{C s^*} \leq \exp\{ C s^* \log(3d/\varepsilon)\},
\end{equation}
whenever $\varepsilon \in (0,1)$. Now, we state an useful lemma on evaluating minimum eigenvalue on $\varepsilon$-net.
\begin{lemma}
\label{lemma: minimum sparse eigen on net}
Let $A$ be a $m \times m$ symmetric positive-definite matrix and $\varepsilon \in (0,1)$. Then, for $\varepsilon$-net $\cN_\varepsilon$ of $\;\bbS_0^{d-1}(s)$ constructed in greedy way, we have
\[
\phi_{\min}(s; A) \geq \min_{u \in \cN_\varepsilon} u^\top A u - 3\varepsilon \phi_{\max}(s; A).
\]
\end{lemma}
The proof of the lemma is deferred to Appendix \ref{sec: proof minimum sparse eigen on net}. Note that from Equation \eqref{eq: sparse eigen concentration 2} and an union bound argument we get 
\begin{equation}
\label{eq: sparse min eigen concentration}
\pr \left(\min_{u \in \cN_\varepsilon} v^\top \widehat{\Sigma}_t v\geq \frac{1}{4K \xi} \right) \geq 1 - 2 \exp\{- c_2 \kappa^2(\xi, \vartheta, K)t + C s^* \log K + Cs^* \log(3d/\varepsilon)\}.
\end{equation}
If $\phi_{\max}(C s^*, \widehat{\Sigma}_t)$ is bounded with high probability, then for  small $\varepsilon$, then along with Lemma \ref{lemma: minimum sparse eigen on net} we will readily have an uniform lower bound on $\phi_{\min}(C s^*, \widehat{\Sigma}_t)$.
\\
\textbf{Bounding $\phi_{\max}(C s^*, \widehat{\Sigma}_t)$:}
Here we gain start with $v \in \cN_\varepsilon$. Similar, to previous discussion we have
\begin{equation*}
     v^\top \widehat{\Sigma}_t v = \frac{1}{t} \sum_{\tau=1}^t \{v^\top x_{a_\tau}(\tau) \}^2  \leq \frac{1}{t}\sum_{\tau =1}^t \max_{i \in [K]}\{v^\top x_{i}(\tau)\}^2.
 \end{equation*}
 
 We define $W_{\tau,v} \triangleq  \max_{i \in [K]}\{v^\top x_{i}(\tau)\}^2$ and note that for a fixed $v\in \bbS_0^{d-1}(C s^*)$, the random variables $\{W_{\tau, v}\}_{\tau =1}^t$ are i.i.d. across the time points. Moreover, due to Assumption \ref{assumptions: contexts}\ref{item: context_subgaussian} and Lemma \ref{lemma: sub-exp and sub-gauss}, we have
 \[
 \norm{W_{\tau, v}}_{\psi_1} = \norm{\max_{i \in [K]} \{v^\top x_i(\tau)\}^2}_{\psi_1} \leq c_4 K \vartheta^2.
 \]
 
 Thus, $\{W_{\tau, v}\}_{\tau = 1}^t$ are i.i.d sub-exponential random variables. Recall, that $\phi_{\max}(C s^*, \Sigma_i) \leq \phi_u$ for all $i \in [K]$. The next lemma provides an upper bound on the moment generating function (MGF) of sub-Exponential random variables.

 \begin{lemma}{\citep[Lemma 2.8.1]{vershynin2018high}}
\label{lemma: MGF sub-exp}
Let $X$ be a mean-zero, sub-Exponential random variable. Then there exists positive constants $c_5, c_6$, such that for any $\lambda$ with $\abs{\lambda} \leq c_5 / \norm{X}_{\psi_1}$, the following is true:
\[
\bbE\{\exp(\lambda X)\} \leq \exp(c_6 \lambda^2 \norm{X}_{\psi_1}^2).
\]
\end{lemma}

Equipped with the above lemma we have the following;

\begin{equation}
\label{eq: concentration max eigen}
    \begin{aligned}
    \pr \left( \frac{1}{t} \sum_{\tau=1}^t W_{\tau,v} - \phi_u \geq \delta\right) & = \pr \left( \sum_{\tau =1}^t \{ W_{\tau,v} - \phi_u\} \geq \delta t\right)\\
    & = \pr \left(\exp \left\{ \mu \sum_{\tau =1}^t ( W_{\tau,v} - \phi_u)\right\} \geq e^{\mu \delta t} \right)\\
    & \leq e^{- \mu\delta t}  \prod_{\tau =1}^t \bbE \{e^{\mu (W_{\tau,v} - \phi_u)} \}.
    \end{aligned}
\end{equation}
For a fixed $\tau \in [t]$ we note the following;
\begin{equation*}
\begin{aligned}
     \bbE \{e^{\mu (W_{\tau,v} - \phi_u)}\} &\leq \sum_{i = 1}^K \bbE\{ e^{\mu [(v^\top x_i(\tau))^2 - \phi_u]}\}\\
     & \leq \sum_{i = 1}^K \bbE\{e^{\mu [(v^\top x_i(\tau))^2 - v^\top \Sigma_i v]}\}.
\end{aligned}     
\end{equation*}
For brevity let $\kappa_i \triangleq \norm{\{v^\top x_i(\tau)\}^2}_{\psi_1}$. If we choose $\mu \leq c_5/\max_{i \in [K]} \kappa_i$, then by Lemma \ref{lemma: MGF sub-exp} we have

\begin{equation*}
   \bbE \{e^{\mu (W_{\tau,v} - \phi_u)}\} \leq  \exp \left( c_6 \mu^2 \max_{i \in [K]} \kappa_i^2 + \log K \right).
\end{equation*}
Using the above inequality in Equation \eqref{eq: concentration max eigen}, it follows that
\begin{equation}
    \label{eq: concentration max eigen 2}
    \pr \left( \frac{1}{t} \sum_{\tau=1}^t W_{\tau,v} - \phi_u \geq \delta\right) \leq \exp \left( - \mu \delta t + c_6 t \mu^2 \max_{i \in [K]} \kappa_i^2 + t\log K\right).
\end{equation}

The right hand side of Equation \eqref{eq: concentration max eigen 2} is minimized at $\mu = \delta /(2 c_2 \max_{i\in [K]} \kappa_i^2)$ with the minimum value of 

\[
\exp \left( - \frac{\delta^2 t}{4 c_6 \max_{i \in  [K]} \kappa_i^2} + t \log K \right).
\]
If $\delta /(2 c_6 \max_{i\in [K]} \kappa_i^2) > c_5/ \max_{i \in [K]} \kappa_i^2 $, then the right hand side of Equation \eqref{eq: concentration max eigen 2} is minimized at $\mu = c_5/ \max_{i \in [K]} \kappa_i^2$ and we get 
\[
\pr \left( \frac{1}{t} \sum_{\tau=1}^t W_{\tau,v} - \phi_u \geq \delta\right) \leq \exp \left( - \frac{c_5 \delta t}{\max_i \kappa_i} + c_6 c_5^2 t + t \log K \right).
\]
Using the fact that $\delta /(2 c_6 \max_{i\in [K]} \kappa_i^2) > c_5/ \max_{i \in [K]} \kappa_i^2$, the right hand side of the above display can be upper bounded by
\[
\exp \left( - \frac{c_5 \delta t}{ 2 \max_{i} \kappa_i} + t \log K\right).
\]
Thus we have for all $\delta >0$
\begin{equation}
    \label{eq: concentration max eigen 3}
    \pr \left( \frac{1}{t} \sum_{\tau=1}^t W_{\tau,v} - \phi_u \geq \delta\right) \leq \exp \left( - \min\left\{ \frac{\delta^2 t}{4 c_6 \max_{i \in  [K]} \kappa_i^2}, \frac{c_5 \delta t}{2 \max_i \kappa_i} \right\} + t \log K\right).
\end{equation}
Next we set $\delta = c_7 \vartheta^2 \phi_u \log K $ for sufficiently large $c_7>0$. Then Equation \eqref{eq: concentration max eigen 3} yields

\begin{equation*}
    \label{eq: concentration max eigen 4}
    \pr \left( \frac{1}{t} \sum_{\tau=1}^t W_{\tau,v} - \phi_u \geq \delta\right) \leq \exp \left( -  c_8 t \log K\right).
\end{equation*}
Finally taking union bond over all vectors in $\cN_\varepsilon$ we get
\begin{equation}
\label{eq: concentration max eigen 5}
    \pr \left( \forall v \in \cN_\varepsilon : v^\top \widehat{\Sigma}_t v \geq  c_9 \vartheta^2 \phi_u \log K \right) \leq \exp\left\{ - c_8 t \log K + C s^* \log K + Cs^* \log(3d/\varepsilon)\right\}.
\end{equation}
Next, to prove the same for all $v \in \bbS_0^{d-1}(C s^*)$ we need the following lemma.
\begin{lemma}[maximum sparse eigenvalue on net]
\label{lemma: maximum sparse eigen on net}
Let $A$ be a $m \times m$ symmetric positive-definite matrix and $\varepsilon \in (0,1/3)$. Then, for $\varepsilon$-net $\cN_\varepsilon$ of $\;\bbS_0^{d-1}(s)$ constructed in greedy way, we have
\[
\max_{v \in \cN_\varepsilon }v^\top A v \leq \max_{v \in \bbS_0^{d-1}(C s^*)} v^\top A v \leq \frac{1}{1 - 3 \varepsilon} \max_{ v\in \cN_\varepsilon} v^\top A v.
\]
\end{lemma}
The proof of the above lemma is deferred to Appendix \ref{sec: proof maximum sparse eigen on net}. Now we set some $\varepsilon \in (0,1/3)$. In light of the above lemma we immediately have that
\begin{equation}
    \label{eq: concentration max eigen 6}
    \pr\left( \phi_{\max}(Cs^*; \widehat{\Sigma}_t) \geq \frac{c_9 \vartheta^2 \phi_u \log K}{1 - 3 \varepsilon}\right) \leq \exp\left\{ - c_8 t \log K + C s^* \log K + Cs^* \log(3d/\varepsilon)\right\}.
\end{equation}
Finally using Equation \eqref{eq: sparse min eigen concentration}, \eqref{eq: concentration max eigen 6} and Lemma \ref{lemma: minimum sparse eigen on net} we have
\begin{equation}
\begin{aligned}
 &\pr \left( \phi_{\min}(C s^*; \widehat{\Sigma}_t)\geq \frac{1}{4K \xi} - \frac{3 \varepsilon c_9\vartheta^2 \log K}{1 - 3 \varepsilon}  \phi_u\right)\\
 &\geq  1 - 2 \exp\{- c_2 \kappa^2(\xi, \vartheta, K)t + C s^* \log K + Cs^* \log(3d/\varepsilon)\}\\
 &\quad - \exp\left\{ - c_8 t \log K + C s^* \log K + Cs^* \log(3d/\varepsilon)\right\}.
\end{aligned}
\end{equation}
Now the result follows from taking $\varepsilon = \min\{1/4,  1/(24 \phi_u  \vartheta^2\xi K \log K + 3)\}$.
\end{proof}

\subsection{Proof of part \ref{item: empirical restrited eigen}}
\label{sec: proof part 2}
In this section we will show that the matrix $\widehat{\Sigma}_t$ enjoys the compatibility condition \eqref{eq: restricted eigen value cond} with high probability. This is equivalent to showing that the quantity
\[
\Psi(S^*; \widehat{\Sigma}_t) \triangleq \inf_{\delta \in \bbC_7(S^*)} \left(\frac{\delta^\top \widehat{\Sigma}_t \delta }{\norm{\delta}_1^2 } \right) s^*
= \phi_{\comp}(S^*; X_t)^2.
\]
is bounded away from 0 with high probability.
First we present the Transfer lemma \citep[Lemma 5]{oliveira2013lower} below.
\begin{lemma}[Transfer lemma]
\label{lemma: transfer lemma}

Suppose $\widehat{\Sigma}_t$ and $\Sigma$ are matrix with non-negative diagonal entries, and assume $\eta\in (0,1)$, $m \in [d]$ are such that 
\begin{equation}
\label{eq: transfer lemma condition}
    \text{$\forall v \in \bbR^d$ with $\norm{v}_0 \leq m, v^\top \widehat{\Sigma}_t v \geq (1-\eta) v^\top \Sigma v$.}
\end{equation}
Assume $D$ is a diagonal matrix whose diagonal entries $D_{j,j}$ are non-negative and satisfy $D_{j,j} \geq (\widehat{\Sigma}_t)_{j,j} - (1-\eta)\Sigma_{j,j}  $. Then
\begin{equation}
    \label{eq: transfer lemma result}
    \textbf{$\forall x \in \bbR^d, x^\top \widehat{\Sigma}_t x \geq (1- \eta) x^\top \Sigma x - \frac{\norm{D^{1/2}x}_1^2
}{m-1}$}.
\end{equation}
\end{lemma}

Condition \eqref{eq: transfer lemma condition} basically demands that $\widehat{\Sigma}_t$ enjoys SRC condition with the sparsity parameter $m$. Then under the proper choice of diagonal matrix $D$ with sufficiently large diagonal elements $\{D_{j,j}\}_{j=1}^d$, Equation \eqref{eq: transfer lemma result} will yield the desired compatibility condition for $\widehat{\Sigma}_t$. We formally state the result in the following lemma:

\begin{proposition}[Empirical compatibility condition]
\label{prop: restricted eigenvalue empirical}
Assume the conditions of Proposition \ref{prop: empirical SRC} hold. Also assume that Assumption \ref{assumptions: contexts}\ref{item: sparse_eigenvalue} holds with $C  = C_0 \phi_u \vartheta^2 \xi K \log K$ for some sufficiently large universal constant $C_0>0$. Then there exists a positive constant $C_1$ such that the following is true:

\begin{align*}
&\pr \left( \Psi(S^*; \widehat{\Sigma}_t) \geq \frac{1}{ C_1\xi K } \right)\\
& =  1 - 2 \exp\{- c_2 \kappa^2(\xi, \vartheta, K)t + C s^* \log K + Cs^* \log(3d/\varepsilon)\}\\
 &\quad - 2\exp\left\{ - c_8 t \log K + C s^* \log K + Cs^* \log(3d/\varepsilon)\right\}
\end{align*}
with $\varepsilon = \min\{1/4,  1/(\tilde{c} \phi_u \vartheta^2 \xi K \log K + 3)\}$ for the same universal constant $\tilde{c}>0$ in Proposition \ref{prop: empirical SRC} and  $\kappa(\xi, \vartheta, K) = \min \{(4 c_3 K \xi \vartheta^2)^{-1}, 1/2\}$.
\end{proposition}

\begin{proof}
As suggested before we will make use of Lemma \ref{lemma: transfer lemma}. Towards this, we set $\Sigma = \frac{1}{4K\xi} \bbI_d$ and $D = \diag(\widehat{\Sigma}_t)$. Next, we define the following two events:
\[
\cG_{t,1} := \left\{ \phi_{\max} (Cs^*; \widehat{\Sigma}_t) \leq \frac{c_9 \vartheta^2 \phi_u \log K}{1 - 3 \varepsilon}\right\},
\]
\[
\cG_{t,2} := \left\{\phi_{\min}(C s^*; \widehat{\Sigma}_t)\geq \frac{1}{8K \xi}\right\},
\]
where the constants $\varepsilon$ and $c_9$ are same as in Proposition \ref{prop: empirical SRC}. Under $\cG_{t,1}$ and $\cG_{t,2}$, the inequality in Equation \eqref{eq: transfer lemma condition} holds with $\eta = 1/2$ and $m = C s^*$. Also, due construction of $D$, we trivially have
\[
D_{j,j} \geq (\widehat{\Sigma}_t)_{j,j} - (1-\eta)\Sigma_{j,j}.
\]
Lastly, note that 
\begin{equation}
    \label{eq: max diag}
    \max_{j \in [d]}D_{j,j} = \max_{j \in [d]} (\widehat{\Sigma}_t)_{j,j}  = \max_{j \in [d]} e_j^\top \widehat{\Sigma}_t e_j \leq \frac{c_9 \vartheta^2 \phi_u \log K}{1 - 3 \varepsilon}
\end{equation}
under $\cG_{t,1}$. Equipped with Lemma \ref{lemma: transfer lemma},
under $\cG_{t,1}$ and $\cG_{t,2}$, for all $x \in \bbC_7(S^*) \cap \bbS^{d-1}$ we have the following:
\begin{equation*}
    \begin{aligned}
     x^\top \widehat{\Sigma}_t x &\geq \frac{1}{8 K \xi} - \frac{\norm{D^{1/2} x}_1^2}{C s^* -1}\\
     & \geq \frac{1}{8K\xi} - \frac{ \left(\frac{c_9 \vartheta^2 \phi_u \log K}{1 - 3 \varepsilon}\right) \norm{x}_1^2 }{C s^* -1}\\
     & \geq \frac{1}{8K \xi} -  \frac{ \left(\frac{c_9 \vartheta^2 \phi_u \log K}{1 - 3 \varepsilon}\right) 64 s^* }{C s^* -1}.
    \end{aligned}
\end{equation*}
The last inequality follows form the fact that
\begin{equation}
    \label{eq: l1 norm break}
    \norm{x}_1 = \norm{x_{S^*}}_1 + \norm{x_{(S^*)^c}}_1 \leq 8 \norm{x_{S^*}}_1 \leq 8 \sqrt{s^*} \norm{x_{S^*}}_2 \leq 8 \sqrt{s^*}.
\end{equation}
Thus, if $C \gtrsim \frac{\phi_u \vartheta^2 \xi  K \log K}{1 - 3 \varepsilon} + \frac{1}{s^*}$ then $x^\top \widehat{\Sigma}_t x \geq 1/(16 K \xi)$. Also, note that from the choice of $\varepsilon$ in Proposition \ref{prop: empirical SRC}, we have $\varepsilon < 1/4$. This further tells that if $C = C_0 \phi_u  \vartheta^2 \xi K \log K$ for large enough $C_0>0$, then 
$$
\inf_{\delta \in \bbC_7(S^*)} \frac{\delta^\top \widehat{\Sigma}_t \delta}{\norm{\delta}_2^2} \geq \frac{1}{16 K \xi}.
$$
Then, the result follows from Proposition \ref{prop: empirical SRC}. Using this and Equation \eqref{eq: l1 norm break} we also have

\[
\Psi(S^*; \widehat{\Sigma}_t) \geq 
\frac{1}{64}
\inf_{\delta \in \bbC_7(S^*)} \frac{\delta^\top \widehat{\Sigma}_t \delta}{\norm{\delta}_2^2} \geq 
\frac{1}{ C_1 \xi K  }
\]
where $C_1 = 1024$. Finally, the result follows from conditioning over the events $\cG_{t,1}$ and $\cG_{t,2}$ and using Proposition \ref{prop: empirical SRC}.

\end{proof}

\subsection{Proof of part \ref{item: Bayesian concentration}}
\label{sec: proof part 3}
In this section we will establish the desired regret bound in Theorem \ref{thm: regret LSB}. The main tools that has been used to prove the regret bound is the Bayesian contraction in high-dimensional linear regression problem. In particular, we will use Theorem \ref{thm: posterior contraction} to control the $\ell_1$-distance between $\tilde{\beta}_t$ and $\beta^*$ at each time point $t \in [T]$. 

We recall that $X_{t} = (x_{a_1}(1), \ldots, x_{a_t}(t))^\top$. Also, note that the sequence $\{x_{a_\tau}(\tau)\}_{\tau =1}^t$ forms an adapted sequence of observations, i.e., $x_{a_\tau}(\tau)$ may depend on the history $\{x_{a_u}(u), r(u)\}_{u=1}^{\tau -1}$. Also, recall that $\{\epsilon(\tau)\}_{\tau=1}^t$ are mean-zero $\sigma$-sub-Gaussian errors.

\begin{lemma}[Bernstein Concentration]
\label{lemma: Bernstein}
Let $\{D_k, \cF_k\}_{k=1}^\infty$ be a martingale difference sequence, and suppose that $D_k$ is a $\sigma$-sub-Gaussian in adapted sense, i.e., for all $\alpha \in \bbR, \bbE[e^{\alpha D_k} \mid \cF_{k-1}] \leq e^{\alpha^2 \sigma^2/2}$ almost surely. Then, for all $t\geq 0$, $\pr \left( \abs{\sum_{k=1}^t D_k} \geq \delta\right) \leq 2 \exp\{- \delta^2/(2 t \sigma^2)\}$.
\end{lemma}

\begin{proof}
Proof of Lemma \ref{lemma: Bernstein} follows from Theorem 2.19 of \cite{wainwright2019high} by setting $\alpha_k = 0$ and $\nu_k = \sigma$ for all $k$.
\end{proof}
Lemma \ref{lemma: Bernstein} is the main tool that is used to control the correlation between $\epsilon_t : = (\epsilon(1), \ldots, \epsilon(t))^\top$ and the chosen contexts $X_{t}$ which is important to control the Bayesian contraction of the posterior distribution in each round. To elaborate, let $X_{t}^{(j)}$ be the $j$th column for $j \in [d]$ and define $D_{t,j} := \epsilon(t) x_{a_t,j}(t)$. Note that for a fixed $j \in [d]$, $\{D_{\tau,j}\}_{\tau=1}^t$ forms a martingale difference sequence with respect to the filtration $\{\cF_\tau\}_{\tau =1}^{t-1}$ with $\cF_\tau := \sigma(\cH_{\tau })$ is the $\sigma$-algebra generated by $\cH_\tau$ and $\cF_1 = \emptyset$. Also note that 

\[
\bbE(e^{\alpha D_{t,j}} \mid \cF_{t-1}) \leq \bbE\{ e^{\alpha^2 \sigma^2 x_{a_t, j}^2(t) / 2}\} \leq \bbE\{ e^{\alpha^2 \sigma^2 \xmax^2/2}\}. 
\]
Thus, using Lemma \ref{lemma: Bernstein}, we have the following proposition:

\begin{proposition}[Lemma EC.2, \cite{bastani2020online}]
\label{prop: error-design correlation}
Define the event
\[
\cT_t(\lambda_0(\gamma)) :=  \left\{ \max_{j \in [d]} \frac{ \abs{\epsilon_t^\top X_{t}^{(j)}}}{t} \leq \lambda_0(\gamma) \right\},
\]
where $\lambda_0(\gamma) =  \xmax \sigma \sqrt{(\gamma^2 +  2\log d)/ t}$. Then we have $\pr \left\{ \cT_t(\lambda_0(\gamma)) \right\} \geq 1 - 2 \exp( - \gamma^2/2)$.
\end{proposition}
The proof of the above proposition mainly relies on the martingale difference structure and Lemma \ref{lemma: bernstein}. It is important to mention that the proof does not depend on any particular algorithm.

For notational brevity we define $\norm{X_{t}} : = \max_{j \in [d]} \sqrt{(X_{t}^\top X_{t})_{j,j}}$. 
Next, we will set $\gamma = \gamma_t : =\sqrt{ 2 \log t}$. Hence by Proposition \ref{prop: error-design correlation} we have $\pr \left\{ \cT_t(\lambda_0(\gamma_t))\right\} \geq 1 - 2 t^{-1}$. Also recall that, under $\cG_{t,1}$ and $\cG_{t,2}$, we have 

\begin{equation}
\label{eq: norm_X bound}
    \frac{1}{\sqrt{8 K \xi}}\leq \norm{X_{t}}/\sqrt{t} \leq \sqrt{4c_9 \phi_u \vartheta^2 \log K}.
\end{equation}
Also, under $\cG_{t,2}\cap \cT_t(\lambda_0(\gamma_t))$ it follows that 
\begin{equation}
    \label{eq: feature-noise corr}
    \begin{aligned}
    \max_{j \in [d]} \frac{\abs{\epsilon_t^\top X_{t}^{(j)}}}{\sigma} & \leq  \xmax \sqrt{2t(\log d + \log t)}= \overline{\lambda}_t
    \end{aligned}
\end{equation}
Now, we are ready to present the proof of the main regret bound.

\subsubsection*{Main regret bound}
Recall the definition of regret is $R(T) = 
\sum_{t = 1}^T \Delta_{a_t}(t)$, where $\Delta_{a_t}(t) = x_{a_t^*}(t)^\top \beta^* - x_{a_t}(t)^\top \beta^*$.
Next, we partition the whole time horizon $[T]$ in to two parts, namely $\{t : 1\leq t \leq T_0\}$ and $\{t : T_0 \leq t \leq T\}$, where $T_0$ will be chosen later. Thus, the regret can be written as 
$$ 
R(T) = \underbrace{\sum_{t=1}^{T_0} \Delta_{a_t}(t)}_{R(T_0)} + \underbrace{\sum_{t = T_0+1}^T \Delta_{a_t}(t)}_{\tilde{R}(T)}.
$$
All notations for expectation operators and probability measures are given in Appendix \ref{sec: proof posterior contraction}.

Now by Assumption \ref{assumptions: contexts}\ref{item: context_bound} and \ref{assumptions: arm-separation}\ref{item: beta-l1} we have the following inequality: 
\begin{equation}
    \label{eq: R(T_0) bound}
    \bbE\{ R(T_0) \} \leq 2 \xmax \bmax T_0.
\end{equation}
Next, we focus on the term $\tilde{R}(T)$. First, we define a few quantities below:

\begin{equation}
\label{eq: uniform compatibility numbers}
\begin{aligned}
&\overline{\phi}_t(s) : = \inf_\delta \left\{ \frac{\norm{X_{t} \delta}_2 \abs{S_\delta}^{1/2}}{t^{1/2} \norm{\delta}_1} : 0 \neq \abs{S_\delta} \leq s\right\},\\
&\widetilde{\phi}_t(s) : =  \inf_\delta \left\{ \frac{\norm{X_{t} \delta}_2 }{t^{1/2} \norm{\delta}_2} : 0 \neq \abs{S_\delta} \leq s\right\}.
\end{aligned}
\end{equation}

Now set 
\[
\overline{\psi}_t(S) = \bar{\phi}_t \left( \left( 2 + \frac{40}{A_4} + \frac{128 A_4^{-1}\xmax^2}{\Psi(S, \widehat{\Sigma}_t)} \right) \abs{S} \right) ,
\]

\[
\widetilde{\psi}_t(S) = \widetilde{\phi}_t \left( \left( 2 + \frac{40}{A_4} + \frac{128 A_4^{-1} \xmax^2}{\Psi(S, \widehat{\Sigma}_t)} \right) \abs{S} \right).
\]
Note that $\bar{\phi}_t(s) \geq \widetilde{\phi}_t(s)$, hence $\bar{\psi}_t(S) \geq \widetilde{\psi}_t(S)$. 

Recall that 
\begin{equation}
    \label{eq: lambda range}
    5 \overline{\lambda}_t/3 \leq \lambda_t \leq 2 \overline{\lambda}_t,
\end{equation}
under $\cG_{t,1}$ and $\cG_{t,2}$. Also define the following events:
\[
\cE_t := \left\{ \norm{\tilde{\beta}_{t+1} - \beta^*}_1 \leq   Q_4 \sigma\xmax   K \xi  (D_* + s^*) \sqrt{\frac{\log d + \log t}{t}}\right\}.
\]
where $D_* = \{1 + (40/A_4) + 128 A_4^{-1}  \xmax^2/\Psi(S^*, \widehat{\Sigma}_t)\}s^*$ and $Q_4$ is large enough universal constant as specified in Theorem \ref{thm: posterior contraction}.
Also we have

\[
\bar{\psi}_t(S)  \leq \bar{\phi}_t \left( \left( 2 + \frac{40}{A_4} + \frac{64 A_4^{-1}\xmax^2}{\Psi(S, \widehat{\Sigma}_t)} \frac{\lambda}{\bar{\lambda}} \right) \abs{S} \right), \text{and}
\]

\[
\widetilde{\psi}_t(S)  \leq \widetilde{\phi}_t \left( \left( 2 + \frac{40}{A_4} + \frac{64 A_4^{-1}\xmax^2}{\Psi(S, \widehat{\Sigma}_t)} \frac{\lambda}{\bar{\lambda}} \right) \abs{S} \right).
\]

 Next, by Proposition \ref{prop: restricted eigenvalue empirical}, the event 
\[
\cG_{t,3}:= \left\{ \Psi(S^*; \widehat{\Sigma}_t) \geq \frac{1}{C_1  \xi K }\right\}
\]
holds with probability of at least  $1 - 2 \exp\{- c_2 \kappa^2(\xi, \vartheta, K)t + C s^* \log K + Cs^* \log(3d/\varepsilon)\}
 \quad - 2\exp\left\{ - c_8 t \log K + C s^* \log K + Cs^* \log(3d/\varepsilon)\right\}$. For, notational brevity, we define
 
 $$\tilde{C} := C_1  \xi K .$$
 Also, define the event 
 \[
 \cG_{t,4}:= \left\{\widetilde{\psi}_t^2(S^*) \geq \frac{1}{8 K \xi} \right\}.
 \]
Noting that $\widetilde{\psi}_t^2(S^*) = \phi_{\min}(\tilde{C}_1 s^*; \widehat{\Sigma}_t)$ with 
\[
\tilde{C}_1 = 2 + \frac{40}{A_4} + 128 A_4^{-1}\xmax^2 \tilde{C},
\]
an argument similar to the proof of Proposition \ref{prop: empirical SRC} yields
 \begin{equation}
    \label{eq: psi_tilde}
    \begin{aligned}
    \pr\left( \cG_{t,4}\right) \geq & \; 1-  2 \exp\{- c_2 \kappa^2(\xi, \vartheta, K)t + \tilde{C}_1 s^* \log K + \tilde{C}_1 s^* \log(3d/\varepsilon)\}\\
& - \exp\left\{ - c_8  \log (K) t + \tilde{C}_1 s^* \log K + \tilde{C}_1 s^* \log(3d/\varepsilon)\right\}.
    \end{aligned}
\end{equation}


Finally, Using Proposition \ref{prop: error-design correlation} with $\gamma = \gamma_d$ and the result of Theorem \ref{thm: posterior contraction}, we have the following:

\begin{equation}
    \label{eq: prob_Et^c}
    \begin{aligned}
    \pr(\cE_t^c) & = \bbE_{X_{t}} \bbE_t^X \left( \ind_{\cE_t^c}\right)\\
     & = \bbE_{X_{t}} \bbE_{t,\br_t}^X\left\{ \Pi_t^X \left(\cE_t^c \mid  \br_t \right)\right\} \\
    & = \bbE_{X_{t}} \bbE_{t,\br_t}^X\left\{ \Pi_t^X \left(\cE_t^c \mid  \br_t \right)\ind_{\cT_t(\lambda_0(\gamma_t))\cap \cG_{t,2}} \right\} + \bbE_{X_{t}} \bbE_{t,\br_t}^X\left\{ \Pi_t^X \left(\cE_t^c \mid  \br_t \right) \ind_{\cT_t^c(\lambda_0(\gamma_t)) \cup \cap \cG_{t,2}^c}\right\} \\
    & \leq \frac{M_1}{d^{s^*}} + \frac{2}{t} + 2 \exp\{- c_2 \kappa^2(\xi, \vartheta, K)t + C s^* \log K + Cs^* \log(3d/\varepsilon)\} \\
    & \quad + \exp\left\{ - c_8 t \log K + C s^* \log K + Cs^* \log(3d/\varepsilon)\right\}
    \end{aligned}
\end{equation}
for some large universal constant $M_1>0$.
Next, let $\cG_t = \cap_{i =1}^4 \cG_{t,i}$. Under the event $\cE_t \cap \cG_t$, we have
\[
D_* +s ^*\leq \underbrace{\left(2 + \frac{40}{A_4} + \frac{128 C_1  K \xi \xmax^2 }{A_4}\right)}_{:= \rho} s^* ,
\]
and,
\[
\norm{\tilde{\beta}_{t+1} - \beta^*}_1 \leq M_2 \rho \sigma \xmax   \xi K  \left\{\frac{ s^{*2}(\log d + \log t)}{t}\right\}^{1/2},
\]
where  $M_2$ is an universal constant depending on $A_4$.
Now we set 
\[
\delta_t =  M_2 \rho \sigma \xmax^2  \xi K  \left\{\frac{ s^{*2}(\log d + \log t)}{t}\right\}^{1/2}.
\]
It follows that under $\cE_{t-1} \cap \cG_{t-1}$, we have the following almost sure inequality:
\begin{align*}
    \Delta_{a_t}(t)  &= x_{a_t^*}^\top(t) \beta^* - x_{a_t}^\top (t) \beta^*\\
    & = x_{a_t^*}^\top (t) \beta^*  - x_{a_t^*}^\top(t) \tilde{\beta}_t + \underbrace{(x_{a_t^*}^\top(t) \tilde{\beta}_t -  x_{a_t}^\top (t) \tilde{\beta}_t)}_{\leq 0} +  x_{a_t^*}^\top (t) \tilde{\beta}_t - x_{a_t}^\top (t) \beta^*\\
    & \leq \norm{x_{a_t^*}(t)}_\infty \Vert \tilde{\beta}_t - \beta^*\Vert_1 + \norm{x_{a_t} (t)}_\infty \Vert \tilde{\beta}_t - \beta^*\Vert_1\\
    & \leq 2 \delta_{t-1}.
\end{align*}

Finally define the event 
\[
\cM_t : = \left\{ x_{a_t^*}^\top \beta^* > \max_{i \neq a_{t}^*} x_{a_t}^\top \beta^* + h_{t-1} \right\}.
\]
Under $\cM_t \cap \cE_{t-1} \cap \cG_{t-1}$, we have the following for any $i \neq a_t^*$:
\begin{align*}
    x_{a_t^*}^\top(t) \tilde{\beta}_t - x_i^{\top}(t) \tilde{\beta}_t & = \innerprod{x_{a_t^*}(t), \tilde{\beta}_t - \beta^*} + \innerprod{x_{a_t}(t) - x_i(t), \beta^*} + \innerprod{x_i(t), \beta^* - \tilde{\beta}_t}\\
    & \geq - \delta_{t-1} + h_{t-1} -  \delta_{t-1}.
\end{align*}
Thus, if we set $h_{t-1} = 3  \delta_{t-1}$ then $ x_{a_t^*}^\top(t) \tilde{\beta}_t - \max_{i \neq a_t^*}x_i^{\top}(t) \tilde{\beta}_t \geq \delta_{t-1}$. As a result, in $t$th round the regret is 0 almost surely as the optimal arm will be chosen with probability 1.
Thus, finally using Assumption \ref{assumptions: arm-separation}\ref{item: margin-cond}, we have

\begin{equation}
    \label{eq: expected summand}
    \begin{aligned}
     \bbE(\Delta_{a_t}(t)) & = \bbE\{ \Delta_{a_t}(t) \ind_{\cM_t^c}\}\\
     & = \bbE\{ \Delta_{a_t}(t) \ind_{\cM_t^c \cap \cE_{t-1} \cap \cG_{t-1}}\} + \bbE\{ \Delta_{a_t}(t) \ind_{\cM_t^c \cap (\cE_{t-1} \cap \cG_{t-1})^c}\}\\
     & \leq 2 \delta_{t-1} \pr (\cM_t^c) + 2 \xmax \bmax \pr(\cE_t^c \cup \cG_t^c) \\
     & \leq 2  \delta_{t-1} \pr(\cM_t^c) + 
     \frac{2 M_1 \xmax \bmax}{d^{s^*}} + \frac{2 \xmax \bmax}{d} \\
     & \quad  + M_3 \xmax \bmax \exp\{- c_2 \kappa^2(\xi, \vartheta, K)t + D s^* \log K + D s^* \log(3d/\varepsilon)\}\\
    & \quad + M_4 \xmax \bmax\exp\left\{ - c_8  \log (K) t + D s^* \log K + D s^* \log(3d/\varepsilon)\right\},
    \end{aligned}
\end{equation}
where $M_3, M_4$ are large enough universal positive constants and $D = \max \{C, \tilde{C}_1\} = \Theta(\phi_u \vartheta^2 \xi K \log K)$. Thus, if we set 
\begin{equation}
\label{eq: T_0}
T_0 = M_5 \max \left\{ \frac{1}{\kappa^2(\xi, \vartheta, K)}, \frac{1}{\log K} \right\}(D s^* \log K + D s^* \log(3d/\varepsilon)),
\end{equation}
for some large universal constant $M_5>0$. Thus, we have

\[
\bbE\{\tilde{R}(T)\} \leq  \underbrace{2 \sum_{t = T_0 + 1}^ T \delta_{t-1} \pr(\cM_t^c)}_{I_\omega} + M_6 \xmax \bmax \exp \left\{- M_7 (D s^* \log K + D s^* \log(3d/\varepsilon)) \right\} + O(\log T).
\]
Recall that 
\[
\delta_t =  M_2 \rho\sigma \xmax^2  \xi K  \left\{\frac{ s^{*2}(\log d  + \log t)}{t}\right\}^{1/2}.
\]

For $\omega \in [0,1]$ we have
\begin{equation}
    \label{eq: I_w 0-1}
    \begin{aligned}
     I_\omega & \leq  2 \sum_{t = T_0 + 1}^ T \delta_{t-1} \left(\frac{3 \delta_{t-1}}{\Delta_*}\right)^\omega\\
     & \asymp \frac{\left\{ 3 M_2 \rho\sigma \xmax^2  \xi K   \right\}^{1+\omega} s^{* 1+ \omega} }{\Delta_*^\omega} \int_{T_0}^T (\log d + \log u)^{\frac{1+ \omega}{2}}u^{- \frac{1 + \omega}{2}}\; du\\
     & \lesssim  \begin{cases}
     \frac{\left [ 3 M_2 \rho\sigma \xmax^2  \xi K    \right]^{1+\omega} s^{* 1+ \omega} (\log d)^{\frac{1+ \omega}{2}} T^{\frac{1 - \omega}{2}}}{\Delta_*^\omega}, & \text{for}\; \omega \in [0,1),\\
     
     \frac{\left [ 3 M_2 \rho\sigma \xmax^2  \xi K \right]^{2} s^{* 2} (\log d + \log T) \log T}{\Delta_*},  & \text{for}\; \omega =1.
     
     \end{cases}
    \end{aligned}
\end{equation}
For $\omega \in (1, \infty) $ we have
\begin{equation}
    \label{eq: I_w 1-infty 1}
    \begin{aligned}
     I_\omega \leq 2 \sum_{t = T_0 + 1}^ T \delta_{t-1} \min \left\{1 , \left(\frac{3 \delta_{t-1}}{\Delta_*}\right)^\omega\right\}.
    \end{aligned}
\end{equation}
Note that 

\[
\frac{3 \delta_{t-1}}{\Delta_*} \leq 1 \Rightarrow t \geq T_1 : =  \left [ 3 M_2 \rho\sigma \xmax^2  \xi K    \right]^2 \frac{s^{*2} \log d }{\Delta_*^2} + 1.
\]
Thus. from Equation \eqref{eq: I_w 1-infty 1} we have 
\begin{align*}
    I_\omega &\leq 2 \sum_{t = T_0 + 1}^{T_1} \delta_{t-1}  + 2\sum_{t = T_1+1}^T \delta_{t-1} \left( \frac{3 \delta_{t-1}}{\Delta_*}\right)^\omega \\
    & \leq 2 \int_{T_0}^{T_1}M_2 \rho\sigma \xmax^2  \xi K   \left\{\frac{ s^{*2}(\log d + \log u)}{u}\right\}^{1/2}\; du + 
     \quad 2  \sum_{t = T_1+1}^T \delta_{t-1} \left( \frac{3 \delta_{t-1}}{\Delta_*}\right)^\omega\\
     & \lesssim 4 \left[M_2 \rho\sigma \xmax^2  \xi K   \right]^2 \left\{ \frac{s^{*2}  (\log d + \log T)}{\Delta_*} \right\} + 2 J_\omega,
\end{align*}
where $J_\omega :=  \sum_{t = T_1+1}^T \delta_{t-1} \left( \frac{3 \delta_{t-1}}{\Delta_*}\right)^\omega$.

Finally, we give bound on the term $J_\omega$:

\begin{equation}
    \label{eq: J_w}
    \begin{aligned}
     J_\omega & = \sum_{t = T_1 + 1}^T \delta_{t-1} \left( \frac{3 \delta_{t-1}}{\Delta_*}\right)^\omega\\
     & = \sum_{t =2}^T \delta_{t-1} \left( \frac{3 \delta_{t-1}}{\Delta_*}\right)^\omega \ind \{3 \delta_{t-1}/\Delta_* \leq 1 \}\\
     & \leq  \left(\dfrac{ 3^{\frac{\omega}{1+\omega}}M_2 \rho\sigma \xmax^2  \xi K  }{\Delta_*^{ \frac{\omega}{1 + \omega}}}\right)^{1 + \omega}  \int_{1}^T \{s^{*2} (\log d + \log u)\}^{\frac{1+\omega}{2}} u^{- \frac{1 + \omega}{2}} \ind\{ u \geq T_1\}\; du\\
     & \leq  \left(\dfrac{ 3^{\frac{\omega}{1+\omega}}M_2 \rho\sigma \xmax^2  \xi K  }{\Delta_*^{ \frac{\omega}{1 + \omega}}}\right)^{1 + \omega} \{s^{*2} (\log d + \log T)\}^{\frac{1+\omega}{2}} \int_{T_1}^\infty u^{- \frac{1 + \omega}{2}} \; du\\
      & =   2 \left(\dfrac{ 3^{\frac{\omega}{1+\omega}}M_2 \rho\sigma \xmax^2  \xi K  }{\Delta_*^{ \frac{\omega}{1 + \omega}}}\right)^{1 + \omega} \{s^{*2} (\log d + \log T)\}^{\frac{1+\omega}{2}} \frac{T_1^{- \frac{\omega -1}{2}}}{\omega -1}\\
      &  \asymp \left\{ \frac{6 \left[M_2 \rho\sigma \xmax^2  \xi K  \right]^2 }{(\omega - 1)}\right\} \left(\frac{s^{*2 } \log d}{\Delta_*}\right).
    \end{aligned}
\end{equation}
    Finally, for $\omega = \infty$ it is easy to see that $J_\omega = 0$. Hence, the result follows from combing Equation \eqref{eq: R(T_0) bound}, \eqref{eq: I_w 0-1}, \eqref{eq: I_w 1-infty 1} and \eqref{eq: J_w}.


\section{Posterior contraction result}
We briefly describe the probability space under which we are working. Given $\beta$, the bandit environment (along with the specific policy $\pi$) gives rise to the chosen contexts $X_t$ and rewards $\br_t$. Here we note that the chosen contexts depend not only on the arm-specific distributions, but also on the sequence of actions taken under $\pi$ till time $t$. Let $\mathcal{Q}_t$ denote the joint distribution of $(\beta, X_t, \br_t)$ under $\beta\sim\Pi$ (prior) and $(X_t, \br_t)\mid\beta \sim \text{SLCB}_t(\beta, \pi, \cP_{\epsilon})$ where the latter indicates the joint distribution of the observed contexts and rewards (till time $t$) under the SLCB environment with policy $\pi$, true parameter $\beta$ and $\cP_{\epsilon}$ denotes the noise distribution. We work under a likelihood misspecified regime, which we now discuss.

We assume that the true parameter is $\beta^*$ and the observations $(X_t, \br_t)$ is generated from $\cQ_t^*:=\text{SLCB}_t(\beta^*, \pi, \cP_{\epsilon}^*)$, where $\pi$ is the policy given by the TS and $\cP_{\epsilon}^*$ is an arbitrary sub-Gaussian distribution. We denote by $\cQ^{*X}_{t,\br_t}$ the conditional distribution of $\br_t$ given $X_t$ arising from the joint $\cQ_t^*$ and $\bbE_{t,\br_t}^X$ to be the expectation under $\cQ^{*X}_{t, \br_t}$. Furthermore, we denote by $\cQ^*_{X_t}$ the marginal distribution of $X_t$ under the joint $\cQ_t^*$ and $\bbE_{X_t}$ to be the corresponding expectation.

For modelling purpose, we place prior $\Pi$ on $\beta$ and model the likelihood as $(X_t, \br_t)\mid \beta \sim \text{SLCB}_t(\beta, \pi, \cP_{\epsilon})$, where $\cP_{\epsilon}$ is taken to be $\sfN(0,\sigma^2)$. This gives rise to a joint distribution $\cQ_t$, as discussed above. Now, let $\Pi_t^{X}(\cdot \mid \br_t)$ denote the posterior distribution of $\beta$ given all others, i.e. it is the conditional measure of $\beta$ given $X_t, \br_t$ arising from the joint $\mathcal{Q}_t$. 

Thus, given a measurable set $B$, $\Pi_t^X(B|\br_t)$ is a random measure, whose randomness is due to $(X_t, r_t)$. In the following result, we consider $\bbE_{t,\br_t}^X \Pi_t^X (B|\br_t)$, which is the expectation of the above under $\cQ^{*X}_{t,\br_t}$. Thus, this quantity itself is a random variable, whose randomness is due to $X_t$. The following result shows that, for $B$ taken as the complement of an appropriate ball around the true $\beta^*$, this random variable is small, almost surely $\cQ_{X_t}^*$.

\label{sec: proof posterior contraction}
\begin{theorem}
\label{thm: posterior contraction}
Consider the bandit problem in \eqref{eq: base model} and let Assumption \ref{assumptions: contexts}-\ref{assumptions: noise} hold. Also, assume that the prior on parameter $\beta$ is modeled as \eqref{eq: prior} with 
\[
(5/3)\overline{\lambda}_t\leq \lambda \leq 2 \overline{\lambda}_t, \quad \overline{\lambda}_t = \xmax \sqrt{2t(\log d + \log t)}
\]

Then the following is true:

\[
 \bbE_{t, \br_t}^{X} \left\{\Pi_t^X\left( 
\norm{\beta - \beta^*}_1  \geq Q_4 \sigma\xmax K \xi  (D_* + s^*) \sqrt{\frac{\log d + \log t}{t}}
\;\bigg\vert\; \br_t\right) \ind_{\cG_t \cap \cT_t(\lambda_0(\gamma_t))}\right\}\lesssim d^{-s^*},
\]
almost sure $X_t$,
where $Q_4$ is a universal constant and $D_* = D_1 s^* + \frac{D_2 \xmax^2 s^*}{\phi_\comp^2(S^*; X_t)} $ with $D_1 = 1 + (40/A_4)$ and $D_2 = 128 A_4^{-1}$.
\end{theorem}

\begin{proof}
Without loss of generality we assume that $\sigma = 1$ as the bandit reward model can be viewed as 
\[
(\br_t/\sigma) = X_t (\beta^*/\sigma) + (\epsilon_t/\sigma).
\]
In this case $\overline{\lambda}_t = \xmax  \sqrt{2t (\log d + \log t)} =: \overline{\lambda}$. 


Next, define the event 
\[
\cT_0 :=   \left\{ \max_{j \in [d]}  \abs{\epsilon_t^\top X_{t}^{(j)}}\leq \bar{\lambda} \right\}.
\]
By Lemma \ref{prop: error-design correlation} and condition \eqref{eq: feature-noise corr}, it follows that fro any measurable set $\cB\subset \bbR^d$,
\begin{equation*}
    \bbE_{t, \br_t} \Pi_t^X( \cB \mid \br_t) \leq  \left[ \bbE_{t,\br_t} \left\{\Pi_t^X( \cB \mid \br_t) \ind_{\cT_0}\right\}\right]^{1/2} + \frac{2}{t}.
\end{equation*}
Recall that the errors $\epsilon_{t}$ is modeled as isotropic standard Gaussian independent of the features. Thus, conditioned on the matrix $X_{t}$, model likelihood ration takes the following form:
\[
\cL_{t, \beta, \beta^*}(\br_t):= \exp\left\{-\frac{\norm{X_{t}\beta - X_{t}\beta^*}_2^2}{2 } + (\br_t - X_{t}\beta^*)^\top (X_{t} \beta - X_{t}\beta^*)\right\}.
\]

Then by Lemma 2 of \cite{castillo2015bayesian} it follows that
\begin{equation*}
\label{eq: lower bound int likelihood}
    \int \cL_{t, \beta, \beta^*}(\br_t) \; d\Pi(\beta) \geq \frac{\pi_d(s^*)}{p^{2s^*}} e^{- \lambda \norm{\beta^*}_1} e^{-1},
\end{equation*}
where $\Pi$ is given by \eqref{eq: prior}. The only change that is needed in their proof to run the argument in our case is the following upper bound:
\[
\norm{X \beta}_2 \leq \norm{\beta}_1 \norm{X} \leq c_9 \vartheta^2 \phi_u \log K/(1- 3 \varepsilon).
\]
The last inequality follows from the fact that we are on the event $\cG_{t,1}$ by assumption. The rest of the proof follows from the fact that $\lambda \in (5\overline{\lambda}/3, 2 \overline{\lambda})$.

Thus by Bayes's formula it follows that
\begin{equation}
\label{eq: posterior prob upper bound 1}
\begin{aligned}
    \Pi_t^X(\cB \mid \br_t) & = \frac{\int_{\cB} \cL_{t, \beta, \beta^*}(\br_t) \; d\Pi(\beta)}{\int \cL_{t, \beta, \beta^*}(\br_t) \; d\Pi(\beta)}\\
    & \leq \frac{ed^{2s^*}}{\pi_d(s^*)} e^{\lambda \norm{\beta^*}_1 } \int_{\cB} \exp \left\{-\frac{\norm{X_{t}\beta - X_{t}\beta^*}_2^2}{2} + (\br_t - X_{t}\beta^*)^\top (X_{t} \beta - X_{t}\beta^*)\right\}\; d\Pi(\beta).
\end{aligned}
\end{equation}
Using Holder's inequality, we see that on $\cT_0$,
\begin{equation}
\label{eq: error design correlation bound}
\begin{aligned}
    (\br_t - X_t \beta^*)^\top X_{t}(\beta - \beta^*)& = \epsilon_t^\top X_{t}(\beta - \beta^*)\\
    & \leq \norm{\epsilon_t^\top X_t}_\infty \norm{\beta - \beta^*}_1\\
    & \leq \bar{\lambda} \norm{\beta - \beta^*}_1.
\end{aligned}
\end{equation}
Therefore, on the event $\cT_0$, the expected value under $\bbE_{\beta^*}$ of the integrand on the right hand side of \eqref{eq: posterior prob upper bound 1} is bounded above by 
\begin{align*}
    & e^{- (1/2)\norm{X_t(\beta - \beta^*)}_2^2  + \overline{\lambda} \norm{\beta - \beta^*}_1} .
\end{align*}

Thus, we have 
\begin{equation}
\label{eq: posterior prob upper bound}
    \Pi_t^X(\cB \mid \br_t) \ind_{\cT_0} \leq \frac{e p^{2s^*}}{\pi_d(s^*)}  \int_\cB e^{\lambda \norm{\beta}_1 - (1/2)\norm{X_t(\beta - \beta^*)}_2^2  + \overline{\lambda} \norm{\beta - \beta^*}_1
    } \; d\Pi(\beta).
\end{equation}
Now, by triangle inequality,
\begin{equation}
\label{eq: beta decomposition 1}
\begin{aligned}
    \lambda \norm{\beta^*}_1 + \overline{\lambda} \norm{\beta - \beta^*}_1 & = \lambda \norm{\beta^*_{S^*}}_1 + \overline{\lambda} \norm{\beta - \beta^*}_1\\
    & \leq \lambda \norm{\beta^*_{S^*} - \beta_{S^*}}_1 + \lambda \norm{\beta_{S^*}}_1 + \overline{\lambda} \norm{\beta_{S^*} - \beta^*_{S^*}}_1 + \overline{\lambda} \norm{\beta_{S^{*c}}}\\
    & = \lambda \norm{\beta_{S^*}}_1 + \overline{\lambda} \norm{\beta_{S^{*c}}}_1 + (\lambda + \overline{\lambda}) \norm{\beta_{S^*} - \beta^*_{S^*}}_1\\
    & = \lambda \norm{\beta}_1 + (\overline{\lambda} - \lambda) \norm{\beta_{S^{*c}}}_1 + (\lambda + \overline{\lambda}) \norm{\beta_{S^*} - \beta^*_{S^*}}_1 .
\end{aligned}
\end{equation}
\paragraph{Case 1:} Suppose $7 \norm{\beta_{S^*} - \beta^*_{S^*}}_1 \leq \norm{\beta_{S^{*c}}}_1$. Then the following holds:
\begin{align*}
    (\lambda + \overline{\lambda}) \norm{\beta_{S^*} - \beta^*_{S^*}}_1 & =  (\overline{\lambda} - 3 \lambda /4) \norm{\beta_{S^*} - \beta^*_{S^*}}_1 + (7 \lambda/4) \norm{\beta_{S^*} - \beta^*_{S^*}}_1\\
    & \leq (\overline{\lambda} - 3 \lambda /4) \norm{\beta_{S^*} - \beta^*_{S^*}}_1 + (\lambda / 4) \norm{\beta_{S^{*c}}}_1.
\end{align*}
Using the above inequality in \eqref{eq: beta decomposition 1} we get
\begin{equation}
    \label{eq: case 1 beta bound}
    \begin{aligned}
    \lambda \norm{\beta^*}_1 + \overline{\lambda} \norm{\beta - \beta^*}_1 & \leq \lambda \norm{\beta}_1 + (\overline{\lambda} - 3\lambda/4) \norm{\beta_{S^{*c}}}_1 + (\overline{\lambda} - 3 \lambda/4) \norm{\beta_{S^*} - \beta^*_{S^{*}}}_1\\
    & = \lambda \norm{\beta}_1 + (\overline{\lambda} - 3\lambda/4) \norm{\beta - \beta^*}_1 
    \end{aligned}
 \end{equation}
 
 \paragraph{Case 2:} Now assume $7 \norm{\beta_{S^*} - \beta^*_{S^*}}_1 > \norm{\beta_{S^{*c}}}_1$. We again focus on the inequality \eqref{eq: beta decomposition 1}, i.e.,
 \begin{equation}
 \label{eq: case 2 beta decomposition}
     \begin{aligned}
    \lambda \norm{\beta^*}_1 + \overline{\lambda} \norm{\beta - \beta^*}_1 & \leq \lambda \norm{\beta}_1 + (\overline{\lambda} - \lambda) \norm{\beta_{S^{*c}}}_1 + (\lambda + \overline{\lambda}) \norm{\beta_{S^*} - \beta^*_{S^*}}_1 \\
    & =  \lambda \norm{\beta}_1 + (\overline{\lambda} - \lambda) \norm{\beta_{S^{*c}}}_1 + ( \overline{\lambda} - \lambda) \norm{\beta_{S^*} - \beta^*_{S^*}}_1 \\
     & \quad  + 2\lambda \norm{\beta_{S^*} - \beta^*_{S^*}}_1\\
     &  = \lambda \norm{\beta}_1  + ( \overline{\lambda} - \lambda) \norm{\beta - \beta^*}_1  + 2\lambda \norm{\beta_{S^*} - \beta^*_{S^*}}_1\\
     & \leq \lambda \norm{\beta}_1  + ( \overline{\lambda} - 3\lambda/4) \norm{\beta - \beta^*}_1  + 2\lambda \norm{\beta_{S^*} - \beta^*_{S^*}}_1.
\end{aligned}
 \end{equation}
 Finally, by compatibility condition and Young's inequality we get 
 \[
 2 \lambda \norm{\beta_{S^*} - \beta^*_{S^*}} \leq 2\lambda \frac{\norm{X_t(\beta - \beta^*)}_2 s^{* 1/2}}{t^{1/2} \phi_{\comp}(S^*;X_t)} \leq \frac{\norm{X_t(\beta - \beta^*)}_2^2}{2} +  \frac{2 s^* \lambda^2 }{t \phi^2_\comp(S^*;X_t)}.
 \]
 Using the above inequality in \eqref{eq: case 2 beta decomposition} we get 
\begin{equation}
    \label{eq: case 2 beta bound}
    \lambda \norm{\beta^*}_1 + \overline{\lambda} \norm{\beta - \beta^*}_1 \leq \lambda \norm{\beta}_1  + ( \overline{\lambda} - 3\lambda/4) \norm{\beta - \beta^*}_1 +\frac{\norm{X_t(\beta - \beta^*)}_2^2}{2} +  \frac{2 s^* \lambda^2 }{t \phi^2_\comp(S^*;X_t)}.
\end{equation}
 Thus combining \eqref{eq: case 1 beta bound} and \eqref{eq: case 2 beta bound} we can conclude
 \begin{equation}
    \label{eq: final beta bound}
    \lambda \norm{\beta^*}_1 + \overline{\lambda} \norm{\beta - \beta^*}_1 \leq \lambda \norm{\beta}_1  + ( \overline{\lambda} - 3\lambda/4) \norm{\beta - \beta^*}_1 +\frac{\norm{X_t(\beta - \beta^*)}_2^2}{2} +  \frac{2 s^* \lambda^2 }{t \phi^2_\comp(S^*;X_t)}.
\end{equation}
 Using the above result  and recalling that $5 \overline{\lambda}/3 \leq \lambda \leq 2 \overline{\lambda}$, we see that the right hand side of \eqref{eq: posterior prob upper bound} is bounded by 
 \begin{align*}
 \Pi_t^X(\cB \mid \br_t) \ind_{\cT_0} &\leq \frac{e d^{2s^*}}{\pi_d(s^*)} e^{\frac{2 s^* \lambda^2 }{t \phi^2_\comp(S^*;X_t)}} \int_\cB  e^{  \lambda \norm{\beta}_1    -(\lambda/4) \norm{\beta - \beta^*}_1 } \; d\Pi(\beta)
 \end{align*}
 
 \paragraph{Controlling sparsity:}
 For the set $\cB = \{\beta : \abs{S_\beta}> L\}$ and $L\geq s^*$, the above integral can be bounded by
 \begin{align*}
     &\sum_{S: s=\abs{S} >L} \frac{\pi_d(s)}{\binom{d}{s}} \left(\frac{\lambda}{2}\right)^s \int e^{- (\lambda/4)\norm{\beta_S - \beta^*_S}}\; d\beta_S\\
     & \leq \sum_{s = L+1}^\infty \pi_d(s) 4^s\\
     & \leq \pi_d(s^*) 4^{s^*} \left(\frac{4 A_2}{d^{A_4}}\right)^{L+1-s^*} \sum_{j = 0}^{\infty} \left(\frac{4 A_2}{d^{A_4}}\right)^j
 \end{align*}
 
 Thus, we have
 \begin{align*}
 &\bbE^X_{t,\br_t}\left\{ \Pi_t^X(\cB\mid \br_t) \ind_{\cT_0}\right\}\\
 &\lesssim \exp \left\{ 4s^* \log d + \frac{2 s^* \lambda^2 }{t \phi^2_\comp(S^*;X_t)} + s^* \log 4 - (A_4/4)(L+1-s^*) \log d\right\}\\
 & \leq \exp \left\{ 5s^* \log d + \frac{4 s^*  \lambda \overline{\lambda} }{t\phi^2_\comp(S^*;X_t)}  - (A_4/4)(L+1-s^*) \log d\right\}
 \end{align*}
 
 Now recall that $\overline{\lambda}^2 = 2 t \xmax^2  (\log d + \log t) \leq  4 t \xmax^2 \log d $. Using this inequality in the above display we have
 
 \[
 \bbE^X_{t,\br_t}\left\{ \Pi_t^X(\cB\mid \br_t) \ind_{\cT_0}\right\} \lesssim \exp \left\{ 5s^* \log d + \frac{16 s^*  (\lambda /\overline{\lambda}) \xmax^2 \log d }{ \phi^2_\comp(S^*;X_t)}  - (A_4/4)(L+1-s^*) \log d\right\}.
 \]
 
 Thus, setting $L \geq  40s^*/A_4 + s^* + \frac{64 A_4^{-1} s^* \xmax^2 }{\phi^2_\comp(S^*;X_t)} (\lambda /\overline{\lambda}) $ then there exists a universal constant $Q_1>0$ such that 
 \[
 \bbE^X_{t,\br_t}\left\{ \Pi_t^X(\cB\mid \br_t) \ind_{\cT_0}\right\} \leq Q_1 d^{-s^*}.
 \]
 
 \paragraph{Control on prediction:}
 Recall that $\lambda/\overline{\lambda} \leq 2$. Using this and the result in the previous part, we can conclude that the posterior distribution is asymptotically supported on the even $\cB_1 = \left\{\beta : \abs{S_\beta} \leq D_* \right\}$, where 
$D_* = D_1 s^* + \frac{D_2  \xmax^2 s^*}{\phi_\comp^2(S^*; X_t)} $ where $D_1 = 1 + (40/A_4)$ and $D_2 = 128 A_4^{-1}$. By combining \eqref{eq: posterior prob upper bound 1}, \eqref{eq: error design correlation bound} and the inequality $\lambda \norm{\beta^*}_1 \leq 2 \overline{\lambda} \norm{\beta - \beta^*}_1 + \lambda \norm{\beta}_1$ we can conclude that any Borel set $\cB$,
\[
\Pi_t^X(\cB\mid \br_t)\ind_{\cT_0} \leq \frac{ed^{2s^*}}{\pi_d(s^*)}  \int_{\cB} \exp \left\{-\frac{\norm{X_{t}\beta - X_{t}\beta^*}_2^2}{2} + 3 \overline{\lambda} \norm{\beta - \beta^*}_1 + \lambda \norm{\beta}_1\right\}\; d\Pi(\beta).
\]
By the definition in \eqref{eq: uniform compatibility numbers} we have,
\begin{equation}
    \label{eq: pred cotrol beta decomposition}
    \begin{aligned}
     (4-1) \overline{\lambda} \norm{\beta - \beta^*}_1 &\leq \frac{4 \overline{\lambda} \norm{X_t(\beta - \beta^*)}_2 \abs{S_{\beta - \beta^*}}^{1/2}}{t^{1/2}\overline{\phi}_t(\abs{S_{\beta - \beta^*}})} - \overline{\lambda} \norm{\beta - \beta^*}_1\\
     & \leq \frac{1}{4}\norm{X_t(\beta - \beta^*)}_2^2 + \frac{16 \overline{\lambda}^2 \abs{S_{\beta - \beta^*}}}{t \overline{\phi}_t(\abs{S_{\beta - \beta^*}})^2} - \overline{\lambda} \norm{\beta - \beta^*}_1.
    \end{aligned}
\end{equation}
Since $\abs{S_{\beta - \beta^*}} \leq \abs{S_\beta} + s^* \leq D_* + s^*$ on the event $\cB_1$, it follows that
\begin{equation}
    \label{eq: pred control probability upper bound}
    \begin{aligned}
     \Pi_t^X(\cB\mid \br_t)\ind_{\cT_0} &\leq \frac{ed^{2s^*}}{\pi_d(s^*)} e^{16 \overline{\lambda}^2(D_*+ s^*)/( t \overline{\psi}_t (S^*)^2 )}\\
     & \quad \times \int_{\cB} e^{-(1/4)\norm{X_t(\beta - \beta^*)}_2^2 - \overline{\lambda} \norm{\beta - \beta^*}_1 + \lambda \norm{\beta}_1}\; d\Pi(\beta). 
    \end{aligned}
\end{equation}
Now we set $\cB = \cB_2 := \{\beta \in \cB_1: \norm{X_t(\beta - \beta^*)}_2 > L\}$, where $L$ will be chosen shortly.
Recall that $\pi_d(s^*) \geq (A_1 p^{-A_3})^{s^*} \pi_p(0)$. It follows that for set $\cB$, the right hand side of \eqref{eq: pred control probability upper bound} is upper bounded by
\begin{align*}
  &\frac{ed^{2s^*}}{\pi_d(s^*)} e^{16 \overline{\lambda}^2(D_*+ s^*)/( t \overline{\psi}_t (S^*)^2)} e^{-(1/4)L^2} \int e^{- \overline{\lambda} \norm{\beta - \beta^*}_1 + \lambda \norm{\beta}_1} \; d \Pi(\beta)\\
  & \lesssim d^{(2+A_3)s^*} A_1^{-s^*} e^{16 \overline{\lambda}^2(D_*+ s^*)/(t \overline{\psi}_t (S^*)^2)} e^{-(1/4)L^2} \underbrace{\sum_{s=0}^d \pi_d(s^*) 2^s}_{O(1)}.
\end{align*}
Hence by a calculation similar to previous discussion yields that for
\[
\frac{1}{4}L^2 = (3+A_3) s^* \log d + \frac{16 \overline{\lambda}^2 (D_* + s^*)}{t \overline{\psi}_t (S^*)^2} \leq Q_2  \xmax^2 (D_* + s^*) \frac{\log d + \log t}{\overline{\psi}_t (S^*)^2} =: L_*^2,
\]
where $Q_2>0$ is  sufficiently large universal constant, then we have
\[
 \bbE^X_{t,\br_t}\left\{ \Pi_t^X(\cB_2 \mid \br_t) \ind_{\cT_0}\right\} \leq \frac{1}{d^{s^*}}.
\]

\paragraph{Control on estimation:}
Similar to \eqref{eq: pred cotrol beta decomposition} we have
\begin{align*}
\overline{\lambda} \norm{\beta - \beta^*}_1 &\leq \frac{ \overline{\lambda} \norm{X_t(\beta - \beta^*)}_2 \abs{S_{\beta - \beta^*}}^{1/2}}{t^{1/2}\overline{\phi}_t(\abs{S_{\beta - \beta^*}})}\\
 & \leq \frac{1}{4} \norm{X_t(\beta - \beta^*)}_2^2 + \frac{\overline{\lambda}^2 \abs{S_{\beta - \beta^*}}}{t \overline{\phi}_t(\abs{S_{\beta - \beta^*}})^2 }.
\end{align*}
On the event $\cB_2 $, we thus have
\[
\overline{\lambda} \norm{\beta - \beta^*}_1 \leq Q_3 \xmax^2 (D_* + s^*) \frac{\log d + \log t}{\overline{\psi}_t(S^*)^2}.
\]
Finally on the event $\cB_2 \cap \cG_t$ we have $\overline{\lambda} =  \xmax \sqrt{2 t (\log d + \log t)}$ and $\overline{\psi}_t(S^*)^2 \gtrsim ( K \xi  )^{-1}$ and it follows that 
\[
\norm{\beta - \beta^*}_1 \leq Q_4 K \xi \xmax (D_* + s^*) \sqrt{\frac{\log d + \log t}{t}}.
\]
\end{proof}

\section{Technical lemmas}

\subsection{Proof of Lemma \ref{lemma: minimum sparse eigen on net}}
\label{sec: proof minimum sparse eigen on net}
As $A$ is symmetric positive definite matrix, by Cholesky decomposition there exists as lower triangular matrix $L$ such that $A = LL^\top$. Let $v \in \bbS_0^{d-1}(s)$. Then there exists a index set $J$ of size $s$, such that $\supp(v) \subseteq J$. Hence we have $v \in E_J$. Now consider the net $\cN_{\varepsilon, J}$ and let $u$ be the nearest point to $v$ in $\cN_{\varepsilon, J}$. Thus we have $\norm{v-u}_2 \leq \varepsilon < 1$ and $\norm{v-u}_0 \leq s$. Then we have the following:
\begin{equation}
    \label{eq: quad break}
\begin{aligned}
    v^\top A v  &= (v-u)^\top A (v-u)  + 2 (v-u)^\top A u + u^\top A u\\
    & \geq u^\top A u - 3 \varepsilon \phi_{\max}(s; A).
\end{aligned}
\end{equation}
The second inequality follows from the following facts:
\[
\abs{(v-u)^\top A (v-u)} \leq \norm{v-u}_2^2 \phi_{\max}(s; A) \leq \varepsilon \phi_{\max}(s; A),
\]
\[
\begin{aligned}
 \abs{(v - u)^\top A u} &= \abs{(v - u)^\top LL^\top u}\\
 & \leq \sqrt{(v-u)^\top L L^\top (v-u)} \sqrt{u^\top LL^\top u} \quad \\
 & = \sqrt{(v-u)^\top A (v-u)} \sqrt{u^\top A u} \\
 & \leq \varepsilon \phi_{\max}(s;A)
\end{aligned}
\]
Then the result follows from taking infimum over $u$ and $v$ in both sides.

\subsection{Proof of Lemma \ref{lemma: maximum sparse eigen on net}}
\label{sec: proof maximum sparse eigen on net}
The lower bound result is trivial. Hence, we focus on the upper bound part.
As $A$ is symmetric positive definite matrix, by Cholesky decomposition there exists as lower triangular matrix $L$ such that $A = LL^\top$. Let $v \in \bbS_0^{d-1}(s)$. Then there exists a index set $J$ of size $s$, such that $\supp(v) \subseteq J$. Hence we have $v \in E_J$. Now consider the net $\cN_{\varepsilon, J}$ and let $u$ be the nearest point to $v$ in $\cN_{\varepsilon, J}$. Thus we have $\norm{v-u}_2 \leq \varepsilon < 1/3$ and $\norm{v-u}_0 \leq s$.
By a similar argument as in the proof of Lemma \ref{lemma: minimum sparse eigen on net}, we can conclude that
\[
v^\top A v \leq 3 \varepsilon \phi_{\max}(s; A) + \max_{u \in \cN_\varepsilon} u^\top A u.
\]
Thus by taking supremum over $v$ on the left-hand side of the above display we get
\[
(1 - 3 \varepsilon)\phi_{\max}(s; A) \leq  \max_{u \in \cN_\varepsilon} u^\top A u \Longleftrightarrow \phi_{\max}(s;A) \leq \frac{1}{1 - 3 \varepsilon} \max_{ u\in \cN_\varepsilon} u^\top A u.
\]

\section{Pseudo code of VBTS and other tables}
\label{sec: pseudo-code VBTS}
\begin{algorithm}[ht!]
\SetAlgoLined
 Set  $\his_0 = \emptyset$\;
 \For{$t=1,\cdots, T$}{
  \If{$t \leq 1$}{Choose action $a_t$ uniformly over $[K]$\;}
  \Else{
    Compute VB posterior $\tilde{\Pi}_{t-1}$ from $\Pi(\cdot\mid \cH_{t-1})$ 
    using CAVI \; 
  Generate sample $\tilde{\beta}_t\sim \tilde{\Pi}_{t-1}$\;
  Play arm: $a_t = \argmax_{i \in {[K]}} \;x_{i}(t)^\top \tilde{\beta}_t$\;
  }
  Observe reward $r_{a_t}(t)$\;
  Update $\his_{t} \leftarrow \his_{t-1} \cup \{(a_t, r_{a_t}(t), x_{a_t}(t))\}$.
 }
 \caption{Variational Bayes Thompson Sampling}
 \label{alg: VBTS}
\end{algorithm}

\renewcommand{\arraystretch}{1.5}
\setlength{\tabcolsep}{1pt}
\begin{table}[ht!]
\small
    \centering
    \caption{Time comparison among the competing algorithms.}
    \vspace{0.1in}
    \begin{tabular}{ |c|c|c|c| }
    \hline
    \multirow{2}{*}{Type} & \multirow{2}{*}{Algorithm} & \multicolumn{2}{|c|}{Mean time of execution (seconds)} \\
    \cline{3-4}
    & & Equicorrelated & Auto-regressive \\
    \hline
    \hline
    \multirow{4}{5em}{\hspace{5pt}Non-TS } & LinUCB & 15.41 & 16.30 \\ & DR Lasso & 3.12 & 3.13 \\ 
    & Lasso-L1 & 3.57 & 3.59\\ 
    & ESTC & 1.01 & 1.05 \\
    \hline
    \hline
    \multirow{3}{5em}{\hspace{15pt}TS} & LinTS & 1344.39 & 1346.46\\ 
    & BLasso TS & 1511.68 & 1455.53\\
    & \textbf{VBTS} & \textbf{29.33} & \textbf{27.65} \\ 
    \hline
    \end{tabular}
    
    \label{tab: time table 2}
\end{table}

\end{appendices}

\end{document}